\documentclass{article}

\usepackage{microtype}
\usepackage{graphicx}
\usepackage{subfigure}
\usepackage{booktabs} 

\usepackage{hyperref}

\usepackage[accepted]{icml2025}

\usepackage{amsmath}
\usepackage{amssymb}
\usepackage{mathtools}
\usepackage{amsthm}
\usepackage{amsfonts}
\usepackage{mathrsfs}

\usepackage[capitalize,noabbrev]{cleveref}

\theoremstyle{plain}
\newtheorem{theorem}{Theorem}[section]

\newtheorem{lemma}[theorem]{Lemma}
\newtheorem{corollary}[theorem]{Corollary}
\theoremstyle{definition}

\newtheorem{assumption}[theorem]{Assumption}
\theoremstyle{remark}

\usepackage[textsize=tiny]{todonotes}

\icmltitlerunning{Continuous-Time Analysis of Federated Averaging}

\begin{document}

\twocolumn[
\icmltitle{Continuous-Time Analysis of Federated Averaging}



\icmlsetsymbol{equal}{*}

\begin{icmlauthorlist}
\icmlauthor{Tom Overman}{yyy}
\icmlauthor{Diego Klabjan}{zzz}
\end{icmlauthorlist}

\icmlaffiliation{yyy}{Department of Engineering Sciences and Applied Mathematics, Northwestern University, Evanston, IL, USA}
\icmlaffiliation{zzz}{Department of Industrial Engineering and Management Sciences, Northwestern University, Evanston, IL, USA}

\icmlcorrespondingauthor{Tom Overman}{tomoverman2025@u.northwestern.edu}

\icmlkeywords{Machine Learning, ICML}

\vskip 0.3in
]



\printAffiliationsAndNotice{}  

\begin{abstract}
Federated averaging (FedAvg) is a popular algorithm for horizontal federated learning (FL), where samples are gathered across different clients and are not shared with each other or a central server. Extensive convergence analysis of FedAvg exists for the discrete iteration setting, guaranteeing convergence for a range of loss functions and varying levels of data heterogeneity. We extend this analysis to the continuous-time setting where the global weights evolve according to a multivariate stochastic differential equation (SDE), which is the first time FedAvg has been studied from the continuous-time perspective. We use techniques from stochastic processes to establish convergence guarantees under different loss functions, some of which are more general than existing work in the discrete setting. We also provide conditions for which FedAvg updates to the server weights can be approximated as normal random variables. Finally, we use the continuous-time formulation to reveal generalization properties of FedAvg.
\end{abstract}

\section{Introduction}
Federated learning (FL) is a popular privacy-preserving machine learning framework allowing a server to train a central model without accessing data locked on clients. In FL, each client holds a subset of the overall data and is not willing to share this data with the server; however, the clients can send model weights to the server. In horizontal FL studied herein, each client holds a subset of samples but each client has access to all features. A popular algorithm for horizontal FL is FedAvg \cite{pmlr-v54-mcmahan17a}, which is the focus in this work.

FedAvg works by averaging the model weights of all clients periodically. Specifically, each client trains a local model for a specific number of iterations using only the data on the client. Typically, this local training is done using batch stochastic gradient descent. Then, the clients send their local model weights to the server, where the server averages the model weights and sends them back to the clients. Understanding the conditions for FedAvg to converge and what factors influence convergence and generalization is of great interest to federated learning practitioners and has been the focus of significant work \cite{li2019convergence,pmlr-v119-karimireddy20a}.

A major challenge in FL is dealing with non-IID data across clients (heterogeneous setting). The local updates on each client significantly diverge due to the different data distributions they are trained on, resulting in slower convergence. Many improvements to FedAvg exist, such as SCAFFOLD \cite{pmlr-v119-karimireddy20a}, which reduces the impact of the data drift during local training and speeds up training in heterogeneous settings, and RADFed \cite{agg_delayed}, which handles non-IID data by delaying aggregation with specialized redistribution rounds. Despite the abundance of new algorithms being introduced to tackle new challenges in FL, FedAvg remains an important algorithm that is the foundation of most new approaches.

The continuous-time perspective of stochastic gradient descent (SGD) has provided a framework for developing more compact proofs of convergence than the discrete process \cite{NEURIPS2019_9cd78264} and has allowed the study of generalization properties. Specifically, it has been shown through the continuous-time representation of SGD how the learning rate and batch size influence the width of local minima converged to during SGD \cite{jastrzębski2018factors}. This local minima width impacts the model's generalization ability to unseen data \cite{keskar2017on}. The continuous-time representation also allows for studying first exit times of SGD from local minima and how this depends on the width of the local minima \cite{heavytail_sgd}. Our work is focused on developing a continuous-time representation of FedAvg. It is our hope that this formulation provides a framework for new interesting analyses, just as the continuous-time representation did for SGD.

We focus on a theoretical analysis of FedAvg. We formulate a continuous-time representation of FedAvg in the form of a stochastic differential equation (SDE) and use this formulation to prove convergence properties. The convergence proofs are relatively compact and the proof framework may be extended to other FL algorithms. We show convergence of FedAvg to a stationary point for general, non-convex loss functions and demonstrate that it is likely not sufficient for only the server learning rate to decay; it is necessary that the client-side learning rate must decay at certain rates. We show convergence of FedAvg to the global minimum for weakly quasi-convex loss functions. To the best of our knowledge, weak quasi-convexity has not been studied for FedAvg up to this point. Next, we show that the server weight updates in FedAvg can be approximated as normal random variables under certain assumptions, even for heterogeneous data with extensive local updates before averaging. This is a surprising result and can assist in further analyses of FL algorithms. Finally, we use our continuous-time approach with a quadratic single variate form of each clients' loss landscape to determine how different FedAvg hyperparameters affect the trade-off between minimizing expected loss and ability to escape poorly-generalizing local minima.

Our contributions are as follows.
\begin{enumerate}
    \item We are the first to formulate an SDE that models the evolution of \textit{server} weights in continuous time during FedAvg. Existing work \cite{fl_continuous} for modeling distributed learning algorithms uses ordinary differential equations without stochasticity.
    \item Using the continuous-time formulation, we devise convergence proofs of FedAvg in deterministic and stochastic cases. We show convergence under a certain normality assumption to a stationary point for non-convex loss functions and convergence to a global minimum for weakly quasi-convex functions. To the best of our knowledge, no other works in either deterministic or stochastic regimes have shown global convergence of FedAvg for weakly quasi-convex loss functions, which are more general than convex functions and allow for locally concave regions. The new proof framework we provide allows for relatively compact proofs of FedAvg and can inspire compact proofs of other FL algorithms. 
    \item We show that the server weight updates converge in distribution to a normal distribution as the number of clients grow, even in non-IID data settings with many local client iterations without server averaging. This justifies the normality assumption in the convergence results. We demonstrate this through the Lyapunov central limit theorem.
    \item Using a quadratic single variate loss function for each client's loss landscape, we uncover dynamics of how various hyperparameters in FedAvg affect the trade-off between generalization and the optimality gap of the expected loss in the continuous-time setting.
\end{enumerate}

In Section \ref{sec:related_work}, we discuss related work on the analysis of FedAvg in the discrete case and continuous-time analysis of SGD. In Section \ref{sec:formulation}, we formulate the SDE that models the evolution of server weights during FedAvg. In Section \ref{sec:proofs}, we provide convergence guarantees of FedAvg for non-convex and weakly quasi-convex loss functions using the SDE formulation. In Section \ref{normality_discussion}, we provide conditions for which the server weight updates can be approximated as normally distributed; this is a key assumption for the SDE formulation to be an Itô process. In Section \ref{sec:quadratic_approx_analysis}, we analyze the case where each client's local loss function is a quadratic; this allows us to examine how various parameters, such as the client learning rate and the number of local iterations, affect the trade-off between minimizing the loss function and generalization.

\section{Related Work}
\label{sec:related_work}
Extensive work has been done in analyzing stochastic gradient descent (SGD) from the continuous-time perspective. Although theory of the discrete process of stochastic gradient descent has existed for a long time, continuous-time perspectives have added value in understanding the behavior of SGD particularly in the area of generalization \cite{Mandt2015ContinuousTimeLO}. Using the continuous-time perspective, it has been shown that learning rate magnitude and batch size affect the sharpness of the local minima that SGD converges to \cite{jastrzębski2018factors}, and the sharpness of local minima has been linked to generalization ability \cite{keskar2017on}. Furthermore, there has been work in developing concise convergence proofs for the continuous SDE form of SGD \cite{NEURIPS2019_9cd78264}. While the previously mentioned works form a Brownian motion representation with normally-distributed noise, there is also work on forming Lévy processes where the noise is the more general $\alpha$-stable distribution that can have heavy tails \cite{heavytail_sgd}. This Lévy process is used to describe first-exit times of SGD under heavy-tailed noise and demonstrates why SGD prefers converging to wide local minima which generalize better than narrow minima. All of these works take the actual discrete stochastic gradient descent process and form the continuous-time SDE that models the evolution of the weights over time.

Furthermore, the discrete process of FedAvg is well-studied. Convergence has been proven for the convex case, even when the data distribution across different clients is non-IID \cite{li2019convergence}. However, the convergence rate is slower when data is highly non-IID and the number of local iterations before averaging is required to be large due to ``client drift'' \cite{pmlr-v119-karimireddy20a}. 

The case of analyzing federated learning from the continuous-time perspective is less-studied. There has been work on analyzing the continuous-time system of a broad class of distributed optimization algorithms using a control-theory perspective, however this work assumes full gradients and thus has no stochastic component \cite{fl_continuous}. Furthermore, their framework does not work for the FedAvg algorithm specifically because FedAvg cannot be mapped to the double-feedback system they developed. Since most implementations of FedAvg use stochastic gradients on the clients, we focus on developing and analyzing an SDE that models the evolution of the global weights while assuming that the client updates use noisy, stochastic gradients. Recent work \cite{meanfieldgame} analyzes the continuous-time limit of local client updates to make connections to game theory, but does not form a continuous-time representation of the server weights and does not study convergence of the global weights. Existing work \cite{iterate-bias} uses the continuous-time limit of each client's local SGD iterates to uncover iterate bias, but the authors do not form the continuous-time SDE for the evolution of weights on the server, and thus their approach is very different from our approach.

\section{Continuous-Time Formulation}
\label{sec:formulation}
We are given $N$ samples with the loss of sample $i$ being $F_i(w)$ and $w\in\mathbb{R}^d$ are the model weights. Furthermore, we assume the samples are gathered across $Q$ clients in the horizontal federated learning fashion. We refer to each local client component of the overall objective function as $F^k(w)=\frac{1}{N_k}\sum_{i\in I_k}F_i(w)$ where $I_k$ is the set of samples that belong to client $k$ and $N_k=|I_k|$. We state the loss function as $F(w)=\sum_{k=1}^Q p_k F^k(w)$ where $p_k$ is the weight of client $k$, usually set to $p_k = \frac{N_k}{N}$. We denote the copy of model weights on client $k$ as $w^k$. The goal is to solve $\min_w F(w)$.


Client weights are updated using standard SGD, except for iterations where averaging occurs. We can write the evolution of client weights as $w^k_{T} =  w^k_{T-1} - \eta_{k,T-1}\frac{1}{S}\sum_{i \in S_{k,T}} \nabla F_i(w^k_{T-1})$, when $T \neq T_0 + E$, and $w^k_{T} =  \sum_{\hat{k}=1}^Q p_{\hat{k}} \bigg( w_{T-E}^{\hat{k}} +\hat{\eta}_{0,T-1}\big( w^{\hat{k}}_{T-1} - \eta_{\hat{k},T-1}\frac{1}{S}\sum_{i \in S_{\hat{k},T}}\nabla F_i(w^{\hat{k}}_{T-1}) -  w_{T-E}^{\hat{k}} \big)\bigg)$, when $T = T_0 +E$,
where $S$ is the batch size, $S_{\hat{k},T}$ is a random batch drawn from the available samples in $I_{\hat{k}}$ at iteration $T$, $T_0$ is the most recent iteration where averaging occurred, $\eta_{\hat{k},T-1}$ is the learning rate on client $\hat{k}$ and may vary over iterations, $\hat{\eta}_{0,T-1}$ is the global server learning rate which may vary over iterations, and $E$ is the number of local SGD updates before averaging. When $T=T_0+E$ and averaging occurs, the most recent iteration of average is incremeted as $T_0 \leftarrow T_0+E$. It is important to note that the averaging step does not imply that clients have access to the updates of the other clients; this average is actually computed on the server and sent back to each client, and on this time step every client holds the same-valued weights. A common choice of server learning rate is $\hat{\eta}_{0,T-1}=1$, which simplifies this update schedule to $w^k_{T} = w^k_{T-1} - \eta_{k,T-1}\frac{1}{S}\sum_{i \in S_{k,T}} \nabla F_i(w^k_{T-1})$, when $ T \neq T_0 + E$, and $w^k_{T} = \sum_{\hat{k}=1}^Q p_{\hat{k}} \bigg( w^{\hat{k}}_{T-1} - \eta_{\hat{k},T-1}\frac{1}{S}\sum_{i \in S_{\hat{k},T}}\nabla F_i(w^{\hat{k}}_{T-1})\bigg)$, when $T = T_0 +E$.

At iteration $T=T_0+E$, each client performs one more step of SGD, then sends updated weights to the server, where the server averages the updates along with an optional server-side learning rate $\hat{\eta}_{0,T-1}$ that controls how much to update the server weights. While the mathematical expression for the case when $T=T_0+E$ shows gradient passing, this is not the case in an actual implementation where it is equivalent to sending the updated weights after each client's local updates. The server then sends this average back to every client. We study the convergence of global server weights in this work, which we denote as $w_{0,T}$ for iteration $T$. Updates to the server weights only occur every $E$ iterations during the averaging step, otherwise they remain the same as the most recent averaging iteration.

As shown in \cite{jastrzębski2018factors}, the stochastic gradients, $\mathcal{G}_{T}^k$, on the local client updates can be assumed to be normally distributed as follows
\begin{equation*}
    \mathcal{G}_{T}^k = \frac{1}{S}\sum_{i \in S_{k,T+1}} \nabla F_i(w^k_T) \sim \mathcal{N}(\nabla F^k(w_T^k), \Sigma_k(w_T^k) )
\end{equation*}
where $\Sigma_k(w_T^k) = (\frac{1}{S} - \frac{1}{N_k})\frac{1}{N_k-1}\sum_{i=1}^{N_k} (\nabla F_i(w_T^k)-\nabla F^k(w_T^k))(\nabla F_i(w_T^k)-\nabla F^k(w_T^k))^T$. We can think of this as a normal random variable centered around the full gradient of the local loss function.

Therefore, we can write our discrete updates in the time region where no averaging occurs as 
\begin{align*}
    w^k_{T+1} &= w^k_T - \eta_{k,T}\mathcal{G}_{T}^k\\
    &= w^k_T - \eta_{k,T}\mathcal{N}(\nabla F^k(w_T^k), \Sigma_k(w_T^k) )\\
    &= w^k_T - \eta_{k,T}\nabla F^k(w_T^k) + \eta_{k,T}\mathcal{N}(0, \Sigma_k(w_T^k) ).\\
\end{align*}

Expanding this for $E$ time steps, we get
\begin{align*}
    w_{T+E}^k=w_T^k + \sum_{i=0}^{E-1} \eta_{k,T+i}(N_{T+i}^k-G_{T+i}^k)
\end{align*}
where $N^k_{T+i}\sim \mathcal{N}(0,\Sigma_k(w_{T+i}^k))$ and $G_{T+i}^k = \nabla F^k(w_{T+i}^k)$.
Notice that the expressions for $G^k_{T+i}$ and $N^k_{T+i}$ depend on $w^k_{T+i}$ and not $w^k_T$. This results in the full expression being a very complicated recurrence relation.

We can re-index as 
\begin{align*}
    w_T^k = w_{T-E}^k + \sum_{i=0}^{E-1} \eta_{k,T-E+i}(N^k_{T-E+i}-G^k_{T-E+i}).
\end{align*}

We assume that at iteration $T$ (and thus also $T-nE$ for all integers $n$ such that $T-nE \geq 0$), the server aggregates across clients, and we write the aggregated server weights at this time as $w_{0,T}$. Using the averaging technique specified in the FedAvg update schedule we get
\begin{align*}
w_{0,T} &= w_{0,T-E} \\&+ \hat{\eta}_{0,T}\sum_{k=1}^{Q} p_k [\sum_{i=0}^{E-1} \eta_{k,T-E+i}(N^k_{T-E+i}-G^k_{T-E+i})].
\end{align*}

We split the global learning rate, $\hat{\eta}_{0,T}$, into the product of a constant term, $h$, used for lifting the difference equation to continuous-time and the learning rate, $\eta_{0,T}$, which depends on iteration $T$. We write this as $\hat{\eta}_{0,T} = h\eta_{0,T}$. Thus, we rewrite the difference equation as
\begin{align*}
w_{0,T} &= w_{0,T-E} \\&+ h\eta_{0,T}\underbrace{\sum_{k=1}^{Q} p_k [\sum_{i=0}^{E-1} \eta_{k,T-E+i}(N^k_{T-E+i}-G^k_{T-E+i})]}_{A_T}.
\end{align*}

\begin{assumption}
\label{normality}
$A_T$ is a normally distributed random variable such that $A_T \sim \mathcal{N}(M_T,V_T)$. Both $M_T$ and $V_T$ are functions of $w_{T-E}^k$.
\end{assumption}
We further discuss Assumption \ref{normality} in Section \ref{normality_discussion}.
We can then write the evolution of iterates of the global server weights as
\begin{align}
\label{eq:original_index}
    w_{0,T} - w_{0,T-E} = h\eta_{0,T} M_T + h\eta_{0,T}\mathcal{N}(0,V_T)
\end{align}
where $M_T = \mathbb{E}[A_T]$ and $V_T=\mathbb{E}[A_TA_T^T]-M_TM_T^T$.

According to the FedAvg server update schedule, the global server weights only change every $E$ iterations as indexed by $T$, thus $w_{0,T} - w_{0,T-E}$ in (\ref{eq:original_index}) represents the difference in a single update of the server weights. So, we can view this as the discretization of an SDE using the Euler-Maruyama method with a step size of $h$. This discretization is more accurate when $h$ is small. We now form the SDE as
\begin{equation}
\label{final_sde}
    dw_0(t) = \eta_0(t)\hat{M}(w_0(t))dt + \eta_0(t)\sqrt{h}\hat{V}^{1/2}(w_0(t))dB(t)
\end{equation}
where $B(t)$ is a standard Brownian motion, $\hat{M}(w_0(t)) = \mathbb{E}[\hat{A}(w_0(t))]$, $\hat{V}^{1/2}(w_0(t))(\hat{V}^{1/2}(w_0(t)))^T=\hat{V}(w_0(t))$, $\hat{V}(w_0(t))=\mathbb{E}[\hat{A}(w_0(t))\hat{A}(w_0(t))^T]-\hat{M}(w_0(t))\hat{M}(w_0(t))^T$,  
and rewriting $A_T$ in continuous-time as
\begin{equation*}
    \hat{A}(w_0(t))=\sum_{k=1}^{Q} p_k [\sum_{i=0}^{E-1} \eta_{k}(t)(N^k(t,i)-G^k(t,i)))]
\end{equation*} where $N^k(t,i)=\mathcal{N}(0,\Sigma_k(w^k(t,i)))$, $G^k(t,i)=\nabla F^k(w^k(t,i))$, $w^k(t,i) = w^k(t,i-1) - \eta_{k}(t) [G^k(t,i-1)+N^k(t,i-1)]$, and $w^k(t,i=0)=w_0(t)$ for all clients $k=1,...,Q$.

We show in Lemma \ref{drift_lip} that the drift term, $\eta_0(t)\hat{M}(w_0(t))$, in (\ref{final_sde}) is Lipschitz continuous which is an important property for our convergence proofs.

\section{Convergence Proofs}
\label{sec:proofs}
We use the formulation of (\ref{final_sde}) in Section \ref{sec:formulation}, and similar tools used in \cite{NEURIPS2019_9cd78264} to form convergence proofs for FedAvg in various settings. While the techniques used in convergence proofs for continuous-time SGD are similar to our approaches for FedAvg, the case for FedAvg that we show is much more complicated.

\begin{assumption}
\label{smooth}
$F_i(w)$ each are $\mu$-smooth functions. We require each $F_i(w)$ to be twice differentiable across their entire domain and $||\nabla F_i(w) ||_\infty \leq L$ and $||\text{diag}(H)||_\infty \leq L$ over the entire domain, where $H$ is the Hessian of $F_i(w)$. 
\end{assumption}

\begin{assumption}
\label{same_client_learning_rates}
The learning rates on each client are the same, may depend on $t$, and can be written as $\eta_k(t)=\eta(t)$ for all $k=1,2,..,Q$. 
\end{assumption}

As a consequence of Assumption \ref{same_client_learning_rates}, we rewrite 
\begin{align*}
    \hat{A}(w_0(t))&=\eta(t)\underbrace{\sum_{k=1}^{Q} p_k [\sum_{i=0}^{E-1}(N^k(t,i)-G^k(t,i)))]}_{\hat{A}_1(w_0(t))}\\
    &=\eta(t) \hat{A}_1(w_0(t))
\end{align*}
and
\begin{align*}
    \hat{V}(w_0(t))&=\mathbb{E}(\hat{A}(w_0(t))\hat{A}(w_0(t))^T)-\hat{M}(w_0(t))\hat{M}(w_0(t))^T\\
    &= \eta(t)^2\bigg(\underbrace{\mathbb{E}(\hat{A}_1\hat{A}_1^T)-\mathbb{E}[\hat{A}_1]\mathbb{E}[\hat{A}_1]^T}_{\hat{V}_1(w_0(t))}\bigg)\\
    &= \eta(t)^2\hat{V}_1(w_0(t)).
\end{align*}

\begin{assumption}
\label{bounded-variance}
We assume that the variances of the server updates are bounded for all $t$. More precisely, $V^* = \max_{t}||\hat{V_1}(w_0(t))||_S < \infty$, where $||\cdot||_s$ is the spectral norm.
\end{assumption}

\begin{assumption}
    \label{constant_variance}
    We assume $\Sigma_k(w_T^k)=\Sigma_k$, where $\Sigma_k \in \mathbb{R}^{d\times d}$ does not vary over iterations and may be different for each client $k$.
\end{assumption}
Assumption \ref{constant_variance} is similar to the assumptions made in the continuous-time SGD literature \cite{Mandt2015ContinuousTimeLO}. 

\subsection{Non-convex loss functions}
We prove convergence to a stationary point for general, non-convex loss functions.
\begin{theorem}
\label{conv_proof_1}
We assume Assumptions \ref{normality}, \ref{smooth}, \ref{same_client_learning_rates}, \ref{bounded-variance}, and \ref{constant_variance} are met, and the server learning rate $\eta_0(t)=1$. For a random time point $\tilde{t} \in [0,t]$ that follows the distribution $\frac{\eta(\tilde{t})}{\int_{0}^{t}\eta(s)ds}$, we have
\begin{align}
    \mathbb{E}_{\tilde{t},\mathcal{G}}||&\nabla F(w_0(\tilde{t}))||^2 \leq\frac{F(w_0(0))-F(w_0^*)}{E \varphi(t)} \nonumber \\&+ \frac{1}{E\varphi(t)}\int_0^t[ C_1\eta(t')^2 + \frac{h\eta(t')^2V^*L}{2} ]dt'
\end{align}
where $C_1=\frac{E^2L\mu \sum_{k=1}^Qp_k[L+\sqrt{\text{Tr}(\Sigma_k)}]}{2}$, $\varphi(t) = \int_{0}^t \eta(t')dt'$, $w_0^*=\text{argmin}_{w_0} F(w_0)$, and the expectation $\mathbb{E}_{\tilde{t},\mathcal{G}}$ is over the random time point $\tilde{t}$ and stochasticity in gradients $\mathcal{G}$.
\end{theorem}

\begin{proof}
    Proof provided in Section \ref{proof:conv_proof_1}.
\end{proof}
Theorem \ref{conv_proof_1} provides a general expression that can be used to easily derive convergence rates for a variety of learning rates, $\eta(t)$. We obtain more concrete convergence rates for specific choices of $\eta(t)$ and asymptotic rates for intervals of $\eta(t)$ in Corollary \ref{cor:concrete_rates}.
\begin{corollary}
\label{cor:concrete_rates}
For a random time point $\tilde{t} \in [0,t]$ that follows the distribution $\frac{\eta(\tilde{t})}{\int_{0}^{t}\eta(s)ds}$, and with $\eta(t)=\frac{1}{t+1}$, we have
\begin{align*}
    \mathbb{E}_{\tilde{t},\mathcal{G}}||\nabla F(w_0(\tilde{t}))||^2 &\leq \frac{F(w_0(0))-F(w_0^*)}{E}\cdot\frac{1}{\log(t+1)} \\&+ \frac{C_1 + \frac{hV^*L}{2}}{E}\cdot\frac{1}{\log(t+1)},
\end{align*}
\\
for $\eta(t)=\frac{1}{\sqrt{t+1}}$, we have
\begin{align*}
    \mathbb{E}_{\tilde{t},\mathcal{G}}||\nabla F(w_0(\tilde{t}))||^2 &\leq \frac{F(w_0(0))-F(w_0^*)}{E (2\sqrt{t+1}-2)} \\&+ \frac{[C_1 + \frac{hV^*L}{2}]\log(t+1)}{E(2\sqrt{t+1}-2)},
\end{align*}
\\
 and with $\eta(t)=1/(t+1)^b$, we have
    \begin{equation*}
        \mathbb{E}_{\tilde{t},\mathcal{G}}||\nabla F(w_0(\tilde{t}))||^2 \leq
        \begin{cases}
            \mathcal{O}(\frac{1}{t^b}) & 0<b<\frac{1}{2}\\
            \mathcal{O}(\frac{\log(t)}{\sqrt{t}}) & b=\frac{1}{2}\\
            \mathcal{O}(\frac{1}{t^{1-b}}) & \frac{1}{2}<b<1\\
            \mathcal{O}(\frac{1}{\log(t)}) & b=1
        \end{cases}.
    \end{equation*}
\end{corollary}

\begin{proof}
    Proof provided in Section \ref{proof:cor:concrete_rates}.
\end{proof}

Next, we show in Corollary \ref{cor:global_rate} that convergence to a stationary point requires a decreasing learning rate on the clients, $\eta(t)$. Interestingly, a decreasing global server learning rate, $n_0(t)$, by itself is likely not sufficient for convergence. This is because the bound on the expected client drift from Lemma \ref{bounded_drift} in the appendix is dependent on the client-side learning rate $\eta(t)$. 
\begin{corollary}
\label{cor:global_rate}
For a random time point $\tilde{t} \in [0,t]$ that follows the distribution $\frac{\eta(\tilde{t})}{\int_{0}^{t}\eta(s)ds}$ and with constant client learning rate $\eta(t)=\eta_c$ and decreasing global server-side learning rate $\eta_0(t)=1/(t+1)$, we have
\begin{align*}
    \mathbb{E}_{\tilde{t},\mathcal{G}}||\nabla F(w_0(\hat{t}))||^2 &\leq \frac{F(w_0(0))-F(w_0^*) + \eta_c^2hV^*L/2}{E\eta_c \log(t+1)} \\&+ \eta_c C_1.
\end{align*}
\end{corollary}
\begin{proof}
    Proof provided in Section \ref{proof:cor:global_rate}.
\end{proof}
Corollary \ref{cor:global_rate} shows that despite a decreasing server-side learning rate, FedAvg is  only guaranteed to converge to a neighborhood of a stationary point that depends on the size of the constant client-side learning rate $\eta(t)=\eta_c$. The algorithm might still converge to a stationary point since we provide an upper bound.

\subsection{Weakly quasi-convex loss functions}
We prove convergence to the global minimum for weakly quasi-convex loss functions as defined in Assumption \ref{wqc}.
\begin{assumption}
\label{wqc}
Function $F_i$ is weakly quasi-convex if for some $w^*$ and $\tau>0$
\begin{equation*}
    \langle \nabla F_i(w), w-w^*\rangle \geq \tau (F_i(w)-F_i(w^*))
\end{equation*}
holds for every $w$.
\end{assumption}

This class of functions is more general than convex functions that are studied in previous discrete analyses of FedAvg. In fact, weakly quasi-convex functions can have locally concave regions, and it has been shown that some non-convex LSTMs induce weakly-quasi convex loss functions \cite{wqc_lstm}. To the best of our knowledge, this is the first work to provide convergence results of FedAvg for weakly quasi-convex loss functions.

\begin{theorem}
    \label{thm_wqc}
    We assume Assumptions \ref{normality}, \ref{smooth}, \ref{same_client_learning_rates}, \ref{bounded-variance}, \ref{constant_variance}, and \ref{wqc} are met, and the serving learning rate is a constant $\eta_0(t)=\eta_0$. For a random time point $\tilde{t} \in [0,t]$ that follows the distribution $\frac{\eta(\tilde{t})}{\int_{0}^{t}\eta(s)ds}$, we have 
    \begin{align*}
     \mathbb{E}_{\tilde{t},\mathcal{G}}[&(F(w_0(\hat{t}))-F(w_0^*))] \leq \frac{||w_0(0)-w_0^*||}{\tau\varphi(t)} \\&+ \frac{C_2}{\tau\varphi(t)} \int_0^t\bigg[\eta(s)^2\bigg(LE\int_0^s \eta(t') dt' \\&+ \sqrt{h}V^*\sqrt{\int_0^s \eta(t')^2 dt'} \bigg) \bigg]ds
    + \frac{C_3}{\tau\varphi(t)} \int_0^t \eta(s)^2 ds 
\end{align*} where $C_2=\mu E^2\sum_{k=1}^Q p_k[L+\sqrt{\text{Tr}(\Sigma_k)}]$, $C_3=\bigg[ \frac{dh\eta_0^2V^*}{2} + \eta_0C_2||w_0(0) - w_0^*|| \bigg]$, and $w_0^*$ is the global minimum point.
\end{theorem}
\begin{proof}
    Proof provided in Section \ref{wqc_proof}.
\end{proof}

Theorem \ref{thm_wqc} provides a general expression that can be used to easily derive convergence rates for a variety of learning rates, $\eta(t)$. We obtain more concrete convergence rates for a specific choice of $\eta(t)$ in Corollary \ref{cor:concrete_rates_wqc}.
\begin{corollary}
\label{cor:concrete_rates_wqc}
If we choose a random time point $\tilde{t} \in [0,t]$ that follows the distribution $\frac{\eta(\tilde{t})}{\int_{0}^{t}\eta(s)ds}$, the serving learning rate is a constant $\eta_0(t)=\eta_0$, and for $\eta(t)=\frac{1}{t+1}$ we have
\begin{align*}
     \mathbb{E}_{\tilde{t},\mathcal{G}}[(F(w_0(\hat{t}))&-F(w_0^*))] \leq  \frac{\eta_0^2C_2LE }{\tau\eta_0}\cdot\frac{t-\log(t+1)}{t\log(t+1)} \\&+ \frac{||w_0(0)-w_0^*|| + C_3 + \eta_0^2C_2\sqrt{h}V^*}{\tau\eta_0\log(t+1)} \\&= \mathcal{O}\bigg(\frac{1}{\log(t)}\bigg).
\end{align*}
\end{corollary}

\begin{proof}
    Proof provided in Section \ref{proof_cor_wqc}.
\end{proof}

\section{On the Assumption of Normally Distributed Server Updates}
In order for the SDE specified in (\ref{final_sde}) to be an Itô process, we require $A_T$ to be normally distributed. We show in this section that this is a reasonable assumption in many cases, such as in the regime of a very large number of clients.
\label{normality_discussion}

\noindent\textbf{Large IID Client Setting}
Assuming data is independent and identically distributed across clients, the weights $p_k$ should be approximately equal, and assuming bounded variance (Assumption \ref{bounded-variance}), we have the conditions met for the traditional central limit theorem. Thus, as the number of clients $Q$ goes to $\infty$, we have $A_T \xrightarrow{\mathcal{D}} \mathcal{N}(M_T,V_T)$.

\noindent\textbf{Large non-IID Client Setting}
We now show that we can approximate $A_T$ as a normal random variable under certain conditions even if the data is not identically distributed across clients. The proof hinges on the Lyapunov central limit theorem which generalizes the traditional central limit theorem to cases where random variables are not identically distributed \cite{lyapunovclt}.
\begin{assumption}
    \label{assumption:diag_covariance}
    The covariance matrix $\Sigma_k(w_t^k)$ is a diagonal matrix for each client $k=1,...,Q$.
\end{assumption}

\begin{assumption}
\label{assumption:not_too_much_div}
We require 
\begin{align*}
    0 < C &\leq \Big(p_k^2\eta_{k,t}^2\mathbb{E}[(N_{k,i}+R_{k,i})^2]\Big)^2 \\&- \frac{1}{2Q}\sum_{j=1}^{Q}\Big(p_k^2\eta_{k,t}^2\mathbb{E}[(N_{k,i}+R_{k,i})^2])^2 \\&- p_j^2\eta_{j,t}^2\mathbb{E}[(N_{j,i}+R_{j,i})^2])^2\Big)^2 ,
\end{align*}
 holds for all clients $k=1,...,Q$ and coordinates $i=1,..,d$, where $\eta_{k,t}$ is the client-side learning rate on client $k$. Values $N_{k,i}$ and $R_{k,i}$ are the $i$-th coordinates of vectors $N_{k}$ and $R_{k}$, respectively, where $N_k=\sum_{i=0}^{E-1} N^k_{t-E+i}$ and $R_k=\sum_{i=0}^{E-1} (\mathbb{E}[G^k_{t-E+i}] - G^k_{t-E+i})$.
\end{assumption}
Assumption \ref{assumption:not_too_much_div} requires that for all clients indexed by $k$, the square of the second moment of $N_k + R_k$ to be greater than the square of the difference of the second moments of $N_k + R_k$ and $N_j + R_j$ averaged over all clients indexed by $j$. This essentially requires the second moments to not be too different across different clients.
\begin{assumption}
\label{bounded_moments}
    The fourth-order mixed moments of $N_k$ and $R_k$ are bounded for all clients $k=1,...,Q$. More formally, we require $\Big|\mathbb{E}[N_{k,i}^uR_{k,i}^v]\Big|\leq D$ for all clients $k=1,...,Q$,  coordinates $i=1,..,d$, and $0\leq u,v \leq 4$ with $u+v=4$.
\end{assumption}
We note that Assumption \ref{bounded_moments} is guaranteed to hold if $N_k$ and $R_k$ have bounded support, which is usually the case in practice because gradients are typically clipped. 

\begin{theorem}
    \label{lyapunov_thm}
    With assumptions \ref{assumption:diag_covariance}, \ref{assumption:not_too_much_div} and \ref{bounded_moments}, as the number of clients $Q$ goes to $\infty$, we have $A_T \xrightarrow{\mathcal{D}} \mathcal{N}(M_T,V_T)$.
\end{theorem}
\begin{proof}
    The proof uses the Lyapunov Central Limit Theorem, which is more flexible than the standard Central Limit Theorem and allows for non-identically distributed random variables as long as the Lyapunov condition is met. The full proof is provided in Section \ref{lyapunov_proof}.
\end{proof}

Section \ref{sec:quadratic_approx_analysis} further demonstrates the normality assumption being met when each client's loss landscape is a quadratic form.

\section{Analysis in the Quadratic Case}
\label{sec:quadratic_approx_analysis}
\noindent\textbf{Quadratic Client Loss Landscape}
We can show that the server updates in FedAvg are normally distributed if we assume the loss landscape of each client follows a different quadratic form.
\begin{assumption}
\label{quadratic_assump}
    Each client's loss landscape is $F^k(w) = \frac{1}{2}(w-a_k)^TU_k(w-a_k)$ for some $U_k\in\mathbb{R}^{d\times d}$ and $a_k \in \mathbb{R}^d$.
\end{assumption}

With Assumption \ref{quadratic_assump}, the loss landscape on the server is $F(w)=\frac{1}{2}\sum_{k=1}^Q p_k (w-a_k)^TU_k(w-a_k)$. This global loss function can be rewritten as 
\begin{align*}
    F(w) = \frac{1}{2}\bigg(w-a\bigg)^T\bigg(\sum_{k=1}^Qp_kU_k\bigg)\bigg(w-a\bigg),
\end{align*}
where $a=(\sum_{k=1}^Qp_kU_k)^{-1}\sum_{k=1}^Qp_kU_k a_k$, which is clearly another quadratic form.

\begin{theorem}
\label{thm-localmins}
With Assumption \ref{constant_variance} and \ref{quadratic_assump}, a server learning rate of $\eta_0(t)=1$, and a constant client-side learning rate of $\eta_k(t)=\eta$, the local weight vector on client $k$ after $E$ local updates is
\begin{align*}
    w_{t+E}^k &\sim w_{0,t} - \eta \sum_{j=0}^{E-1}(I-\eta U_k)^jU_k(w_{0,t}-a_k) \\&- \eta \sum_{j=0}^{E-1}\sum_{i=1}^j (I-\eta U_k)^{j-i}U_k\mathcal{N}(0,\Sigma_k) \\&+ E\eta\mathcal{N}(0,\Sigma_k),
\end{align*}
where $w_{0,t}$ is the shared server weight vector from the most recent averaging step.
\end{theorem}
\begin{proof}
Proof provided in Section \ref{proof:localmins}.
\end{proof}

Theorem \ref{thm-localmins} shows that no matter how many local updates are performed on a client, the resulting final update to the weights is normally distributed. Therefore, after averaging, the updates to the global weights are also normally distributed.

\textbf{Generalization Analysis of Single Variate Quadratic Loss Landscape}

Using the evolution of local weights for quadratic loss functions provided in Theorem \ref{thm-localmins}, assuming we have the single variate case, and using the same procedure for building the SDE as in Section \ref{sec:formulation}, we obtain the SDE for the global weights as 
\begin{align}
\label{quadratic_approx_sde}
    &dw_0(t) = -\bigg[\sum_{k=1}^Q p_k \sum_{j=0}^{E-1} (1-\eta U_k)^jU_k(w_0(t)-a_k)\bigg]dt \nonumber\\&+ \sqrt{\eta }\bigg[E + \sum_{k=1}^Q\sum_{j=0}^{E-1}\sum_{i=1}^j p_k \sqrt{\Sigma_k} (1-\eta U_k)^{j-i}U_k \bigg]dB(t).
\end{align}

The SDE in (\ref{quadratic_approx_sde}) is a linear SDE and allows us to find analytical solutions as shown in Theorem \ref{quadratic_approx_sol}.

\begin{theorem}
\label{quadratic_approx_sol}
The solution to the SDE specified in (\ref{quadratic_approx_sde}) is normally distributed as
    \begin{align*}
        w_0(t) \sim \mathcal{N}(m_0(t),v_0(t)),
    \end{align*}
    where 
    \begin{align*}
        m_0(t) = C_4 + (w_0(0) - C_4) e^{-(\sum_{k=1}^Qp_k\sum_{j=0}^{E-1}(1-\eta U_k)^j U_k)t},
    \end{align*}
    \begin{align*}
        C_4 = \frac{\sum_{k=1}^Qp_k\sum_{j=0}^{E-1}(1-\eta U_k)^j U_k a_k}{\sum_{k=1}^Qp_k\sum_{j=0}^{E-1}(1-\eta U_k)^j U_k},
    \end{align*}
    and
    \begin{align*}
        v_0(t)&=\frac{-\eta \big[E + \sum_{k=1}^Q\sum_{j=0}^{E-1}\sum_{i=1}^j p_k \sqrt{\Sigma_k} (1-\eta U_k)^{j-i}U_k \big]^2}{2\sum_{k=1}^Qp_k\sum_{j=0}^{E-1}(1-\eta U_k)^j U_k}\\&\cdot\bigg( \exp\big(-(\sum_{k=1}^Qp_k\sum_{j=0}^{E-1}(1-\eta U_k)^j U_k)t\big) - 1 \bigg).
    \end{align*}
\end{theorem}
\begin{proof}
    Proof is provided in Section \ref{proof_quadratic_approx_sol}.
\end{proof}

We are interested in how various hyperparameters influence the expected loss and the variance of the solution distribution shown in Theorem \ref{quadratic_approx_sol}. In the limit as $\eta \rightarrow 0$ and $t \rightarrow \infty$, the solution approaches $\mathcal{N}(\frac{\sum_{k=1}^Q p_k U_ka_k}{\sum_{k=1}^Q p_k U_k},0)$ which is the true global minimum with zero variance. This matches the behavior of SGD, that as the learning rate is decreased, the expected loss approaches the true minimum but the variance of the solution decreases and thus may not be able to escape poor local minima \cite{escape_local_minima,jastrzębski2018factors}.

As the number of local iterations before communication, $E$, grows, the mean of the solution diverges from the true solution. Furthermore, as $E$ grows, the variance of the solution increases as $\mathcal{O}(E^2)$. This causes the expected solution to have a high bias with regards to the true solution, but an increase in $E$ may allow FedAvg to escape poorly-generalizing local minima. This demonstrates the trade-off between expected loss and ability to escape poor local minima.

This reasoning demonstrates how our proposed continuous-time limit of FedAvg can be used to explore behaviors of FL algorithms such as generalization ability. We leave further exploration of these ideas to future work, such as exploring the first exit time of the continuous-time formulation of FedAvg from a local minimum, similar to the existing work in SGD \cite{heavytail_sgd}. 

\section{Conclusion}
We have developed a continuous-time SDE that models the evolution of server weights during FedAvg. Using this formulation, we prove convergence to stationary points for general non-convex loss functions and are the first to show global convergence for weakly quasi-convex loss functions. We discuss the various assumptions required for the SDE to be an Itô process and demonstrate that these assumptions are reasonable. We demonstrate that the updates to the server weights can be approximated as a normal random variable in many cases, even if data is not identically distributed across clients. Finally, we provide a generalization analysis using our continuous-time formulation and quadratic client losses. Our SDE formulation serves as a framework for future efforts in analyzing FedAvg and other FL algorithms from the continuous-time perspective.

\section*{Impact Statement}
This paper is focused on mathematically analyzing the underlying behavior of federated learning from a new perspective.

\nocite{*}

\bibliography{main}
\bibliographystyle{icml2025}

\newpage
\appendix
\onecolumn
\section{Proofs}
For a nice review of the mathematical preliminaries of SDEs, see the starting sections of the Supplementary Materials in \cite{NEURIPS2019_9cd78264}.
\label{proof_section}
\newline
We introduce a bound on the expected drift between client weights and the server weights. Lemma \ref{bounded_drift} is used in the proofs of Theorems \ref{conv_proof_1} and \ref{thm_wqc}. The proof follows similar ideas as \cite{li2019convergence} but has several key differences due to the continuous-time setting.
\begin{lemma}
\label{bounded_drift}
With Assumptions \ref{smooth}, \ref{same_client_learning_rates}, \ref{constant_variance}, the expected drift between server global weights and the weights of client $k$ is bounded as
\begin{equation*}
\mathbb{E}||w_0(t)-w^k(t,i)|| \leq i\eta(t)[L+\sqrt{\text{Tr}(\Sigma_k)}].
\end{equation*}
\end{lemma}
\subsection{Proof of Lemma \ref{bounded_drift}}
\label{proof:bounded_drift}
\begin{proof}
We first examine the form of $w^k(t,i)$ as 
\begin{align*}
    w^k(t,i) = w_0(t) + \sum_{j=0}^{i-1}\eta(t)[N^k(t,j) - G^k(t,j)]
\end{align*}

Now we form the bound
\begin{align*}
\mathbb{E}||w_0(t) - w^k(t,i)|| &= \mathbb{E}||w_0(t) - w_0(t) - \sum_{j=0}^{i-1}\eta[N^k(t,j) - G^k(t,j)]||\\
&=\mathbb{E}||\sum_{j=0}^{i-1}\eta[G^k(t,j) - N^k(t,j)]||\\
&\leq \sum_{j=0}^{i-1}\eta(t)[\mathbb{E}||\nabla F^k(w^k(t,j))|| + \mathbb{E}||N^k(t,j)||]\\
&\leq \sum_{j=0}^{i-1}\eta(t)[L+\sqrt{\text{Tr}(\Sigma_k)}]\\
&=i\eta(t)[L+\sqrt{\text{Tr}(\Sigma_k)}].
\end{align*}
\end{proof}

\subsection{Other Lemmas}
We first include some helpful lemmas that have been proved in other works or are straightforward to show.
\begin{lemma}
\label{spectral_norm}
For two symmetric, $d \times d$ matrices $A$ and $B$, we have the following hold
\begin{equation*}
\text{Tr}(AB) \leq d||A||_S ||B||_S.
\end{equation*}
\end{lemma}
\begin{proof}
\cite{spectral_norm1}, \cite{spectral_norm2}, and \cite{NEURIPS2019_9cd78264} all provide proofs of this statement.
\end{proof}

\begin{lemma}
\label{valid_pdf}
The probability density function $f(t') = \frac{\eta(t')}{\varphi(t)}$ defined over support $t\in[0,t]$ is a valid probability density function.
\begin{proof}
We define the learning rate $\eta(t)$ as strictly positive, and thus the resulting probability density function is non-negative. We also show that
\begin{align*}
\int_{-\infty}^{\infty}f(t')dt' &= \int_{0}^t\frac{\eta(t')}{\varphi(t)}dt'\\
&= \frac{1}{\varphi(t)} \underbrace{\int_{0}^t\eta(t')dt'}_{\varphi(t)}\\
&=1.
\end{align*}
Therefore, the probability density function integrated over its entire support is 1, and the necessary conditions of a probability density function are met.
\end{proof}
\end{lemma}

\begin{lemma}
\label{drift_lip}
The drift term, $\eta_0(t)\hat{M}(w_0(t))$, in Equation \ref{final_sde} is Lipschitz continuous.
\end{lemma}
\begin{proof}
We require Assumption \ref{smooth}.
The drift term is a vector-valued function, and it suffices to prove Lipschitz continuity for each component of the function to prove Lipschitz continuity of the whole vector-valued function \cite{lipschitz_multivar}. Furthermore, it is well known that for everywhere differentiable functions, the function is Lipschitz continuous if and only if the absolute value of the derivative of the function is bounded by a finite value for the entire input domain. We prove this lemma by combining these two facts and proving the bounded derivative, component-wise for $\hat{M}(t)$.

\begin{align*}
    \hat{M}(w_0(t))&=\mathbb{E}\bigg[\sum_{k=1}^{Q} p_k [\sum_{i=0}^{E-1} \eta_k(t)(N^k(t,i)-G^k(t,i)))]\bigg]\\
    &= -\sum_{k=1}^{Q} p_k [\sum_{i=0}^{E-1} \eta_k(t)\mathbb{E}[G^k(t,i)]]\\
    &= -\sum_{k=1}^{Q} p_k [\sum_{i=0}^{E-1} \eta_k(t)\mathbb{E}[\nabla F^k(w^k(t,i))]]
\end{align*} where $N^k(t,i)=\mathcal{N}(0,\Sigma_k(w^k(t,i)))$, $G^k(t,i)=\nabla F^k(w^k(t,i))$, $w^k(t,i) = w^k(t,i-1) - \eta_k(t) [G^k(t,i-1)+N^k(t,i-1)]$, and $w^k(t,i=0)=w_0(t)$ for all $k=1,...,Q$.

We examine at the component-level and index for the $j$-th component. We use the notation $F'(x)=\frac{d}{dx}F(x)$. We also change $F^k()$ to $F_k()$ for readability when using the $'$ derivative notation. We have
\begin{align*}
    \hat{M}(t)_j &= -\sum_{k=1}^{Q} p_k [\sum_{i=0}^{E-1} \eta_k(t)\mathbb{E}[\frac{d}{d w_j^k(t,i)} F_k(w_j^k(t,i))]]\\
    &= -\sum_{k=1}^{Q} p_k [\sum_{i=0}^{E-1} \eta_k(t)\mathbb{E}[F'_k(w_j^k(t,i))]].
\end{align*}

Now we need to show $|\frac{d}{dw^k_j(t,0)}\hat{M}(t)_j|\leq M$ for some finite $M$. Note, the derivative is with respect to $w^k_j(t,0)$, but the function arguments are $w^k_j(t,i)$, so chain rule needs to be applied consecutively. We have

\begin{align*}
    \frac{d}{dw^k_j(t,0)}F'_k(w_j^k(t,i)) &= F_k''(w_j^k(t,i))\cdot\frac{dw_j^k(t,i)}{dw^k_j(t,0)}\\
    &=F_k''(w_j^k(t,i))\cdot\frac{d}{dw^k_j(t,0)}\bigg(w_j^k(t,i-1) - \eta_k(t) [F_k'(w^k_j(t,i-1))+N_j^k(t,i-1)]\bigg)\\
    &=F_k''(w_j^k(t,i))\cdot\bigg(\frac{dw_j^k(t,i-1)}{dw^k_j(t,0)} - \eta_k(t)F_k''(w^k_j(t,i-1))\cdot \frac{dw_j^k(t,i-1)}{dw^k_j(t,0)}\bigg)\\
    &=F_k''(w_j^k(t,i))\cdot\bigg(\frac{dw_j^k(t,i-1)}{dw^k_j(t,0)}[1 - \eta_k(t)F_k''(w^k_j(t,i-1))]\bigg).
\end{align*}

Using a new local iteration, we have that
\begin{align*}
    \frac{dw_j^k(t,i-1)}{dw^k_j(t,0)} = \bigg(\frac{dw_j^k(t,i-2)}{dw^k_j(t,0)}[1 - \eta_k(t)F_k''(w^k_j(t,i-2))]\bigg).
\end{align*}

Combining we get
\begin{align*}
    \frac{d}{dw^k_j(t,0)}F'_k(w_j^k(t,i)) &=F_k''(w_j^k(t,i))\cdot\bigg[\bigg(\frac{dw_j^k(t,i-2)}{dw^k_j(t,0)}[1 - \eta_k(t)F_k''(w^k_j(t,i-2))]\bigg)[1 - \eta_k(t)F_k''(w^k_j(t,i-1))]\bigg].
\end{align*}

Since $\frac{dw_j^k(t,0)}{dw^k_j(t,0)}=1$, we can expand as
\begin{align*}
    \frac{d}{dw^k_j(t,0)}F'_k(w_j^k(t,i)) &=F_k''(w_j^k(t,i))\cdot\bigg[\bigg(\frac{dw_j^k(t,i-2)}{dw^k_j(t,0)}[1 - \eta_k(t)F_k''(w^k_j(t,i-2))]\bigg)[1 - \eta_k(t)F_k''(w^k_j(t,i-1))]\bigg]\\
    &=F_k''(w_j^k(t,i))\cdot\bigg[\prod_{l=1}^{i-1}[1 - \eta_k(t)F_k''(w^k_j(t,i-l))]\bigg].
\end{align*}

Returning to the original expression for $\hat{M}(t)_j$ we have

\begin{align*}
   \frac{d}{dw^k_j(t,0)} \hat{M}(t)_j = -\sum_{k=1}^{Q} p_k [\sum_{i=0}^{E-1} \eta_k(t) F_k''(w_j^k(t,i))\cdot\bigg[\prod_{l=1}^{i-1}[1 - \eta_k(t)F_k''(w^k_j(t,i-l))]\bigg]].
\end{align*}

We now bound the absolute value of this derivative as 
\begin{align*}
   |\frac{d}{dw^k_j(t,0)} \hat{M}(t)_j | &\leq \sum_{k=1}^{Q} p_k \sum_{i=0}^{E-1} \eta_k(t) |F_k''(w_j^k(t,i))|\cdot\bigg[\prod_{l=1}^{i-1}[1 + \eta_k(t)|F_k''(w^k_j(t,i-l))|]\bigg]\\
   &\leq \sum_{k=1}^{Q} p_k \sum_{i=0}^{E-1} \eta_k(t) L\cdot\bigg[\prod_{l=1}^{i-1}[1 + \eta_k(t) L]\bigg]\\
   &\leq \sum_{k=1}^{Q} p_k \sum_{i=0}^{E-1} [1 + \eta_k(t) L]^i.
\end{align*}

This holds for all components, therefore, we can conclude the drift term, $\eta_0(t)\hat{M}(t)$, is Lipschitz continuous.

\end{proof}

\subsection{Proof of Theorem \ref{conv_proof_1}}
\label{proof:conv_proof_1}
\begin{proof}
The proof structure is inspired by the structure of the continuous-time proofs for standard SGD, but have several additional complexities due to the averaging of client weights to determine server weights after a certain amount of time. The main steps are finding a suitable energy function of the stochastic process, bounding the infinitesimal diffusion generator, and using Dynkin's formula to complete the proof.

We first write out $\hat{M}(t)$ as
\begin{align*}
\hat{M}(t)=-\sum_{k=1}^Q p_k (\sum_{i=0}^{E-1} \eta_k(t)\mathbb{E}(G(t-1,i))).
\end{align*}

A good resource for reviewing the background mathematics used in this proof is in the Supplementary materials of \cite{NEURIPS2019_9cd78264}.

Using a similar approach as in \cite{NEURIPS2019_9cd78264}, we define an appropriate energy function and then use Dynkin's formula to complete the proof. We are able to use all of the machinery for Itô diffusions due to our SDE being an Itô diffusion. In particular, we show in Lemma \ref{drift_lip} that the drift term is Lipschitz continuous and the proof for Lemma 1 in the supplementary materials of \cite{NEURIPS2019_9cd78264} proves Lipschitz continuity of the volatility matrix $\eta_0(t)\sqrt{h}\hat{V}^{1/2}(t)$.

We define $\partial_x(\cdot)$ as the vector of first derivatives with respect to each component of $x$ and $\partial_{xx}(\cdot)$ as the matrix of partial derivatives of $\partial_{xx}(\cdot)$ with respect to each component of $x$. They are just the gradient and hessian with respect to the arguments of the energy function, but we use different notation than $\nabla$ to prevent confusion with $\nabla F(w_0)$ which is the gradient of the loss function. The infinitesimal generator for an Itô diffusion of the form $dX(t)=f(t)dt + \sigma(t)dB(t)$ is defined as 
\begin{equation*}
\mathscr{A}(\cdot) = \partial_t(\cdot) + \langle \partial_x(\cdot), f(t) \rangle + \frac{1}{2}\text{Tr}(\sigma(t)\sigma(t)^T\partial_{xx}(\cdot)).
\end{equation*}

Following \cite{NEURIPS2019_9cd78264} and using Itô's lemma and the definition of an Itô diffusion, we obtain Dynkin's formula as
\begin{equation}
\label{dynkins}
\mathbb{E}[\mathcal{E}(X(t),t)] - \mathcal{E}(x_0,0)=\mathbb{E}[\int_0^t \mathscr{A}\mathcal{E}(X(t'),t')dt'].
\end{equation}
Dynkin's formula is vital to our convergence proofs along with defining the correct energy function $\mathcal{E}(\cdot)$.

We choose the energy function $\mathcal{E}(w_0): \mathbb{R}^d \rightarrow \mathbb{R}_+$ as $\mathcal{E}(w_0) = F(w_0)-F(w_0^*)$. It is important to note that we define the energy function as just an argument of $w_0$ rather than both $w_0$ and $t$, thus the first term in Dynkin's formula is not applicable. This follows the same approach as \cite{NEURIPS2019_9cd78264}. 

We then bound the expectation of the stochastic integral of the infinitesimal generator of the process $\{\mathcal{E}(w_0(t))\}_{t\geq 0}$ as
\begin{align*}
\mathscr{A}\mathcal{E}(w_0(t)) = \underbrace{\langle \partial_{w_0}(\mathcal{E}(w_0(t))), \eta_0(t)\hat{M}(t) \rangle }_{B_1} + \underbrace{\frac{1}{2}\text{Tr}(h(\eta_0(t))^2\hat{V}(t)\partial_{w_0w_0}(\mathcal{E}(w_0(t))))}_{B_2}.
\end{align*}

We first bound $B_1$ (for simplicity we assume $\eta_k(t)=\eta(t)$ for all clients) as
\begin{align*}
B_1 &= \langle \nabla F(w_0(t)), -\eta_0(t)\sum_{k=1}^Q p_k\eta(t) (\sum_{i=0}^{E-1} \mathbb{E}[G^k(t,i)]) \rangle \\
&= -\eta_0(t)\sum_{k=1}^Q p_k\eta(t)\sum_{i=0}^{E-1}\nabla F(w_0(t))^T\mathbb{E}[G^k(t,i)]\\
&= -\eta_0(t)\sum_{i=0}^{E-1}\eta(t) \nabla F(w_0(t))^T\sum_{k=1}^Q p_k\mathbb{E}[G^k(t,i)].
\end{align*}

We first define a function $R(t)$ such that $||\nabla F(w_0(t))||^2 = \nabla F(w_0(t))^T(\sum_{k=1}^Q p_k\mathbb{E}[G^k(t,i)]) + R(t)$.

We then bound $R(t)$ as
\begin{align*}
R(t) &= ||\nabla F(w_0(t))||^2 - F(w_0(t))^T(\sum_{k=1}^Q p_k\mathbb{E}[G^k(t,i)])\\
&=\nabla F(w_0(t))^T( F(w_0(t)) - \sum_{k=1}^Q p_k\mathbb{E}[G^k(t,i)] )\\
&\leq |\nabla F(w_0(t))^T( \nabla F(w_0(t)) - \sum_{k=1}^Q p_k\mathbb{E}[G^k(t,i)] )|\\
&\leq \underbrace{||\nabla F(w_0(t))||}_{\leq L}\cdot || \nabla F(w_0(t)) - \sum_{k=1}^Q p_k\mathbb{E}[G^k(t,i)] || \;\;\; \text{(Cauchy-Schwarz)}\\
&\leq L|| \nabla F(w_0(t)) - \sum_{k=1}^Q p_k\mathbb{E}[G^k(t,i)]||\\
&= L ||\mathbb{E}\sum_{k=1}^Q[\nabla F(w_0(t)) - p_k\nabla F^k(w^k(t,i))]||\\
&\leq L \sum_{k=1}^Qp_k||\mathbb{E}[\nabla F(w_0(t)) - \nabla F^k(w^k(t,i))]||\\
&\leq L\sum_{k=1}^Qp_k\mathbb{E}|| \nabla F(w_0(t)) - \nabla F^k(w^k(t,i)) || \;\;\; \text{(Jensen's Inequality)}\\
&\leq L\mu\sum_{k=1}^Qp_k \mathbb{E}|| w_0(t) - w^k(t,i) || \;\;\; \text{($F$ is $\mu$-smooth)}\\
&\leq L\mu i\eta(t)\sum_{k=1}^Qp_k[L+\sqrt{\text{Tr}(\Sigma_k)}] \;\;\; \text{(Lemma \ref{bounded_drift})}.
\end{align*}

It follows that
\begin{align*}
-\nabla F(w_0(t))^T(\sum_{k=1}^Q p_k\mathbb{E}[G^k(t,i)]) &= R(t) - ||\nabla F(w_0(t))||^2\\
&\leq L\mu i\eta(t)\sum_{k=1}^Qp_k[L+\sqrt{\text{Tr}(\Sigma_k)}] - ||\nabla F(w_0(t))||^2.
\end{align*}

We return to our bound on $B_1$ as
\begin{align*}
B_1 &\leq \eta_0(t)\sum_{i=0}^{E-1}\eta(t)\Big[ L\mu i\eta(t)\sum_{k=1}^Qp_k[L+\sqrt{\text{Tr}(\Sigma_k)}] - ||\nabla F(w_0(t))||^2\Big ]\\
&= -E\eta_0(t)\eta(t)||\nabla F(w_0(t))||^2 + \eta_0(t)\eta(t)^2 L\mu \sum_{k=1}^Qp_k[L+\sqrt{\text{Tr}(\Sigma_k)}]  \sum_{i=0}^{E-1} i\\
&= -E\eta_0(t)\eta(t)||\nabla F(w_0(t))||^2 + \eta_0(t)\eta(t)^2 L\mu \sum_{k=1}^Qp_k[L+\sqrt{\text{Tr}(\Sigma_k)}]  (\frac{1}{2}E(E-1))\\
&\leq -E\eta_0(t)\eta(t)||\nabla F(w_0(t))||^2 + \frac{1}{2} E^2\eta_0(t)\eta(t)^2 L\mu \sum_{k=1}^Qp_k[L+\sqrt{\text{Tr}(\Sigma_k)}].
\end{align*}

Now we bound $B_2$ as
\begin{align*}
B_2 &= \frac{1}{2}\text{Tr}(h(\eta_0(t))^2\hat{V}(t)\partial_{w_0w_0}(\mathcal{E}(w_0(t))))\\
&=\frac{h(\eta_0(t))^2}{2}\text{Tr}\bigg(\hat{V}(t)\nabla^2F(w_0(t))\bigg).
\end{align*}

Using Lemma \ref{spectral_norm}, we can bound the trace of a product of symmetric matrices by the product of spectral norms. Both $\hat{V}(t)$ and $\nabla ^2 F(w_0(t))$ are symmetric by their construction. So we have

\begin{align*}
B_2 &\leq \frac{h(\eta_0(t))^2}{2}||\hat{V}(t)||_S ||\nabla^2F(w_0(t))||_S.
\end{align*}

Then by assumptions \ref{smooth}, \ref{same_client_learning_rates}, and \ref{bounded-variance}, we have
\begin{align*}
B_2 &\leq \frac{h(\eta_0(t)\eta(t))^2}{2}||\hat{V}_1(t)||_S ||\nabla^2F(w_0(t))||_S\\
&\leq \frac{h(\eta_0(t)\eta(t))^2V^*L}{2}.
\end{align*}

We return to the bound on the infinitesimal generator as
\begin{align*}
\mathscr{A}\mathcal{E}(w_0(t)) &= \underbrace{\langle \partial_{w_0}(\mathcal{E}(w_0(t))), \eta_0(t)\hat{M}(t) \rangle }_{B_1} + \underbrace{\frac{1}{2}\text{Tr}(h(\eta_0(t))^2\hat{V}(t)\partial_{w_0w_0}(\mathcal{E}(w_0(t))))}_{B_2}\\
&\leq -E\eta_0(t)\eta||\nabla F(w_0(t))||^2 + \underbrace{\frac{E^2L\mu \sum_{k=1}^Qp_k[L+\sqrt{\text{Tr}(\Sigma_k)}]}{2}}_{C_1} \eta_0(t)\eta^2 + \frac{h(\eta_0(t)\eta(t))^2V^*L}{2}.
\end{align*}

Then we use Dynkin's Formula (Equation \ref{dynkins}) and substitute $\eta_0(t)=1$ to obtain the following expression
\begin{equation}
\label{dynkin1}
\mathbb{E}[\mathcal{E}(w_0(t))] - \mathcal{E}(w_0(0)) \leq \mathbb{E}\bigg[\int_0^t[ -E\eta(t')||\nabla F(w_0(t'))||^2 + C_1\eta(t')^2 + \frac{h\eta(t')^2V^*L}{2} ]dt'\bigg].
\end{equation}
We follow a similar approach as \cite{NEURIPS2019_9cd78264} and define $\varphi(t) = \int_{0}^t \eta(t')dt'$. We can then define the distribution defined by the probability density function $f(t') = \frac{\eta(t')}{\varphi(t)}$. Lemma \ref{valid_pdf} shows this is a valid probability density function. We define a random variable $\hat{t} \in [0,t]$ with probability density function $\frac{\eta(t')}{\varphi(t)}$. We then use a similar trick to \cite{lus} and \cite{NEURIPS2019_9cd78264} and use the law of the unconscious statistician to obtain
\begin{equation}
\label{lus_eq}
\mathbb{E}_{\tilde{t}}||\nabla F(w_0(\hat{t}))||^2 = \frac{1}{\varphi(t)}\int_0^t \eta(t')||\nabla F(w_0(t'))||^2 dt'
\end{equation}
where $\mathbb{E}_{\tilde{t}}$ is the expectation over the random time point $\tilde{t}$.

We substitute Equation \ref{lus_eq} into Inequality \ref{dynkin1} and get
\begin{equation}
\label{dynkin2}
\underbrace{\mathbb{E}[\mathcal{E}(w_0(t))]}_{E_1(t)} - \mathcal{E}(w_0(0)) \leq -E\varphi(t)\mathbb{E}_{\mathcal{G},\tilde{t}}||\nabla F(w_0(\hat{t}))||^2 + \int_0^t[ C_1\eta(t')^2 + \frac{h\eta(t')^2V^*L}{2} ]dt'
\end{equation}

where $\mathbb{E}_{\mathcal{G},\tilde{t}}$ is the expectation over the random time point $\tilde{t}$ and the stochastic gradients, $\mathcal{G}$.

We notice that $E_1(t) = F(w_0(t)) - F(w_0^*) \geq 0$, so we can safely drop this term from the inequality.

So we can rewrite Equation \ref{dynkin2} as
\begin{align*}
    \mathbb{E}_{\mathcal{G},\tilde{t}}||\nabla F(w_0(\hat{t}))||^2 &\leq \frac{\mathcal{E}(w_0(0))}{E \varphi(t)} + \frac{1}{E\varphi(t)}\int_0^t[ C_1\eta(t')^2 + \frac{h\eta(t')^2V^*L}{2} ]dt'\\
    &=\frac{F(w_0(0))-F(w_0^*)}{E \varphi(t)} + \frac{1}{E\varphi(t)}\int_0^t[ C_1\eta(t')^2 + \frac{h\eta(t')^2V^*L}{2} ]dt'.
\end{align*}

\end{proof}

\subsection{Proof of Corollary \ref{cor:concrete_rates}}
\label{proof:cor:concrete_rates}
\begin{proof}
We first examine the case where $\eta(t)=\frac{1}{t+1}$. 

We have
\begin{align*}
    \varphi(t) &= \int_{0}^t \eta(t')dt'\\
    &= \int_{0}^t \frac{1}{t'+1}dt'\\
    &= \log(t'+1)\bigg|_{0}^t\\
    &= \log(t+1).
\end{align*}

Therefore we find
\begin{align*}
    \mathbb{E}_{\mathcal{G},\tilde{t}}||\nabla F(w_0(\tilde{t}))||^2 &\leq \frac{F(w_0(0))-F(w_0^*)}{E \varphi(t)} + \frac{1}{E\varphi(t)}\int_0^t[ C_1\eta(t')^2 + \frac{h\eta(t')^2V^*L}{2} ]dt'\\
    &=\frac{F(w_0(0))-F(w_0^*)}{E \log(t+1)} + \frac{1}{E\log(t+1)} \bigg[C_1\underbrace{\int_0^t \eta(t')^2 dt'}_{=\frac{t}{t+1}} + \frac{hV^*L}{2}\int_0^t\eta(t')^2dt'\bigg]\\
    &\leq \frac{F(w_0(0))-F(w_0^*)}{E \log(t+1)} + \frac{C_1 + \frac{hV^*L}{2}}{E\log(t+1)}\\
    &= \frac{F(w_0(0))-F(w_0^*) + C_1 + \frac{hV^*L}{2}}{E}\frac{1}{\log(t+1)}.
\end{align*}

Now we examine $\eta(t)=\frac{1}{\sqrt{t+1}}$
We have $\varphi(t) = \int_{0}^t \eta(t')dt' = 2\sqrt{t+1}-2$, and we find 
\begin{align*}
    \mathbb{E}_{\mathcal{G},\tilde{t}}||\nabla F(w_0(\tilde{t}))||^2 &\leq \frac{F(w_0(0))-F(w_0^*)}{E \varphi(t)} + \frac{1}{E\varphi(t)}\int_0^t[ C_1\eta(t')^2 + \frac{h\eta(t')^2V^*L}{2} ]dt'\\
    &=\frac{F(w_0(0))-F(w_0^*)}{E (2\sqrt{t+1}-2)} + \frac{1}{E(2\sqrt{t+1}-2)} \bigg[C_1\underbrace{\int_0^t \eta(t')^2 dt'}_{=\log(t+1)} + \frac{hV^*L}{2}\int_0^t\eta(t')^2dt'\bigg]\\
    &\leq \frac{F(w_0(0))-F(w_0^*)}{E (2\sqrt{t+1}-2)} + \frac{[C_1 + \frac{hV^*L}{2}]\log(t+1)}{E(2\sqrt{t+1}-2)}.
\end{align*}

Now we examine the asymptotic rates. This proof is the same as in \cite{NEURIPS2019_9cd78264}.
The general form of the inequality is 
\begin{equation*}
     \mathbb{E}_{\mathcal{G},\tilde{t}}||\nabla F(w_0(\tilde{t}))||^2 \leq \frac{F(w_0(0))-F(w_0^*)}{E \varphi(t)} + \frac{1}{E\varphi(t)}\int_0^t[ C_1\eta(t')^2 + \frac{h\eta(t')^2V^*L}{2} ]dt'.
\end{equation*}
The first term on the RHS of the inequality is the deterministic term and has rate $\mathcal{O}(t^{b-1})$ for $0<b<1$ and $\mathcal{O}(1/\log(t))$ for $b=1$. The stochastic term is $\mathcal{O}(t^{-b})$ for $b \in (0,1/2) \cup (1/2,1)$, $\mathcal{O}(\frac{\log(t)}{\sqrt{t}})$ for $b=1/2$, and $\mathcal{O}(\frac{1}{\log{t}})$ for $b=1$.

\end{proof}

\subsection{Proof of Corollary \ref{cor:global_rate}}
\label{proof:cor:global_rate}
\begin{proof}
    We examine the convergence for the choice of client learning rate, $\eta(t)=\eta_c$, and server learning rate, $\eta_0(t)=\frac{1}{t+1}$. Most of the proof follows the same steps as the proof for Theorem \ref{conv_proof_1}, but we return to the inequality of the infinitesimal generator as
    \begin{align*}
    \mathscr{A}\mathcal{E}(w_0(t)) &\leq -E\eta_0(t)\eta||\nabla F(w_0(t))||^2 + \underbrace{\frac{E^2L\mu \sum_{k=1}^Qp_k[L+\sqrt{\text{Tr}(\Sigma_k)}]}{2}}_{C_1} \eta_0(t)\eta^2 + \frac{h(\eta_0(t)\eta(t))^2V^*L}{2}\\
    &=-E\eta_0(t)\eta_c||\nabla F(w_0(t))||^2 + \underbrace{\frac{E^2L\mu \sum_{k=1}^Qp_k[L+\sqrt{\text{Tr}(\Sigma_k)}]}{2}}_{C_1} \eta_0(t)\eta_c^2 + \frac{h(\eta_0(t)\eta_c)^2V^*L}{2}.
    \end{align*}

Use Dynkin's Formula, \ref{dynkins}, to obtain
\begin{equation}
\label{cor:dynkin1}
\mathbb{E}[\mathcal{E}(w_0(t))] - \mathcal{E}(w_0(0)) \leq \mathbb{E}\bigg[\int_0^t[ -E\eta_c\eta_0(t')||\nabla F(w_0(t'))||^2 + C_1\eta_0(t')\eta_c^2 + \frac{h\eta_0(t')^2\eta_c^2V^*L}{2} ]dt'\bigg].
\end{equation}

Now we follow steps similar to the final steps of the proof for Theorem \ref{conv_proof_1}.

We set $\varphi(t) = \int_0^t \eta_0(t')dt'$. We then substitute Equation \ref{lus_eq} into Inequality \ref{dynkin1} and get
\begin{equation}
\label{dynkin2-2}
\underbrace{\mathbb{E}[\mathcal{E}(w_0(t))]}_{E_1(t)} - \mathcal{E}(w_0(0)) \leq -E\eta_c\varphi(t)\mathbb{E}_{\mathcal{G},\tilde{t}}||\nabla F(w_0(\hat{t}))||^2 + \eta_c^2\int_0^t[ C_1\eta_0(t') + \frac{h\eta_0(t')^2V^*L}{2} ]dt'.
\end{equation}

We notice that $E_1(t) = F(w_0(t)) - F(w_0^*) \geq 0$, so we can safely drop this term from the inequality.

So we can rewrite Equation \ref{dynkin2-2} as
\begin{align*}
    \mathbb{E}_{\mathcal{G},\tilde{t}}||\nabla F(w_0(\hat{t}))||^2 &\leq \frac{\mathcal{E}(w_0(0))}{E\eta_c \varphi(t)} + \frac{\eta_c}{E\varphi(t)}\int_0^t[ C_1\eta_0(t') + \frac{h\eta_0(t')^2V^*L}{2} ]dt'\\
    &=\frac{F(w_0(0))-F(w_0^*)}{E \eta_c\varphi(t)} + \frac{\eta_c}{E\varphi(t)}\int_0^t[ C_1\eta_0(t') + \frac{h\eta_0(t')^2V^*L}{2} ]dt'.
\end{align*}

Now we substitute $\eta_0(t)=\frac{1}{t+1}$ and get
\begin{align*}
    \mathbb{E}_{\mathcal{G},\tilde{t}}||\nabla F(w_0(\hat{t}))||^2 &\leq \frac{F(w_0(0))-F(w_0^*)}{E\eta_c \log(t+1)} + \frac{\eta_c}{E\log(t+1)}[ C_1\log(t+1) + \frac{hV^*L}{2} ].\\
    &= \frac{F(w_0(0))-F(w_0^*) + \eta_c^2hV^*L/2}{E \eta_c \log(t+1)} + \frac{C_1\eta_c}{E}\\
    &= \frac{F(w_0(0))-F(w_0^*) + \eta_c^2hV^*L/2}{E\eta_c \log(t+1)} + \frac{E\eta_cL\mu \sum_{k=1}^Qp_k[L+\sqrt{\text{Tr}(\Sigma_k)}]}{2}.
\end{align*}
\end{proof}

\subsection{Proof of Theorem \ref{thm_wqc}}
\label{wqc_proof}
\begin{proof}
This proof follows similarly to the proof of Theorem \ref{conv_proof_1} where we find a suitable energy function, bound the infinitesimal generator, then use Dynkin's Formula to complete the bound.

For this case, we choose the energy function $\mathcal{E}(w_0): \mathbb{R}^d \rightarrow \mathbb{R}_+$ as $\mathcal{E}(w_0) = \frac{1}{2}||w_0-w_0^*||^2$.

We then bound the expectation of the stochastic integral of the infinitesimal generator of the process $\{\mathcal{E}(w_0(t))\}_{t\geq 0}$ as
\begin{align*}
\mathscr{A}\mathcal{E}(w_0(t)) = \underbrace{\langle \partial_{w_0}(\mathcal{E}(w_0(t))), \eta_0(t)\hat{M}(t) \rangle }_{B_1} + \underbrace{\frac{1}{2}\text{Tr}(h(\eta_0(t))^2\hat{V}(t)\partial_{w_0w_0}(\mathcal{E}(w_0(t))))}_{B_2}.
\end{align*}

First examine $B_1$ as
\begin{align*}
    B_1 &= \langle w_0-w_0^*, -\eta_0(t)\sum_{k=1}^Q p_k \eta(t) \sum_{i=0}^{E-1}\mathbb{E}[G^k(t,i)] \rangle \\
    &=-\eta_0(t)\sum_{k=1}^Q p_k\eta(t) \sum_{i=0}^{E-1} (w_0(t)-w_0^*)^T\mathbb{E}[\nabla F^k(w^k(t,i))].
\end{align*}
We define $R(t)$ such that 
\begin{equation}
\label{proof2_int1}
    (w_0(t)-w_0^*)^T\mathbb{E}[\nabla F^k(w^k(t,i))] + R(t) = (w_0(t)-w_0^*)^T(\nabla F^k(w_0(t))).
\end{equation}

Now we bound $R(t)$ as
\begin{align*}
    R(t) &= (w_0(t)-w_0^*)^T(\nabla F^k(w_0(t)) - \mathbb{E}[\nabla F^k(w^k(t,i))])\\
    &\leq |(w_0(t)-w_0^*)^T(\nabla F^k(w_0(t)) - \mathbb{E}[\nabla F^k(w^k(t,i))])|\\
    &\leq ||w_0(t)-w_0^*||\cdot ||\nabla F^k(w_0(t)) - \mathbb{E}[\nabla F^k(w^k(t,i))]|| \;\;\; \text{(Cauchy-Schwarz)}\\
    &\leq ||w_0(t)-w_0^*||\cdot ||\mathbb{E}[\nabla F^k(w_0(t)) - \nabla F^k(w^k(t,i))]||\\
    &\leq ||w_0(t)-w_0^*||\cdot \mathbb{E}||\nabla F^k(w_0(t)) - \nabla F^k(w^k(t,i))||\;\;\; \text{(Jensen's Inequality)}\\
    &\leq \mu||w_0(t)-w_0^*||\cdot \mathbb{E}||w_0(t) -  w^k(t,i)||\;\;\; (\mu\text{-smooth})\\
    &\leq \mu i\eta(t)[L+\sqrt{\text{Tr}(\Sigma_k)}]||w_0(t)-w_0^*|| \;\;\; (\text{Lemma \ref{bounded_drift}})\\
    &= \mu i\eta(t)[L+\sqrt{\text{Tr}(\Sigma_k)}]||w_0(0) + (w_0(t) - w_0(0)) -w_0^*||.
\end{align*}
Now we need to find a bound on $||w_0(0) + (w_0(t) - w_0(0)) -w_0^*||$. Specifically, we need to bound $||w_0(t)-w_0(0)||$.
We bound $||w_0(t)||$ as
\begin{align*}
    ||w_0(t)-w_0(0)|| &= ||w_0(0) - w_0(0) + \int_0^t \eta_0(t') \hat{M}(t') dt' + \int_0^t \sqrt{h}\eta_0(t') \hat{V}^{1/2}(t') dB(t')||\\
    &\leq \int_0^t \eta_0(t')||\hat{M}(t')||dt' + ||\int_0^t \sqrt{h}\eta_0(t') \hat{V}^{1/2}(t') dB(t')||.\\
\end{align*}
We assume constant server learning rate $\eta_0(t)=\eta_0$. We examine $\int_0^t \eta_0(t')||\hat{M}(t')||dt'$ as follows
\begin{align*}
    \int_0^t \eta_0(t')||\hat{M}(t')||dt' &= \int_0^t \eta_0(t')||\eta(t')\sum_{k=1}^Q p_k \sum_{i=0}^{E-1}\mathbb{E}[G^k(t',i)]||dt'\\
    &\leq \int_0^t \eta_0(t')\eta(t')\sum_{k=1}^Q p_k \sum_{i=0}^{E-1}\mathbb{E}||\nabla F^k(w_k(t',i))||dt'\\
    &\leq \int_0^t \eta_0(t')\eta(t')LE dt'\\
    &= \eta_0LE\int_0^t \eta(t') dt'.
\end{align*}

Therefore we have the following
\begin{equation*}
    ||w_0(t) - w_0(0)|| \leq  \eta_0LE\int_0^t \eta(t') dt' + \sqrt{h}\eta_0||\int_0^t  \hat{V}^{1/2}(t') dB(t')||.
\end{equation*}

Returning back to $R(t)$, we have
\begin{equation*}
    R(t) \leq \mu i\eta(t)[L+\sqrt{\text{Tr}(\Sigma_k)}]\bigg[ ||w_0(0) - w_0^*|| + \eta_0LE\int_0^t \eta(t') dt' + \sqrt{h}\eta_0||\int_0^t  \hat{V}^{1/2}(t') dB(t')|| \bigg].
\end{equation*}

Returning to Equation \ref{proof2_int1}, it follows that
\begin{align*}
    -(w_0(t)-w_0^*)^T\mathbb{E}[\nabla F^k(w^k(t,i))] &= R(t) - (w_0(t)-w_0^*)^T(\nabla F^k(w_0(t)))\\
    &\leq \mu i\eta(t)[L+\sqrt{\text{Tr}(\Sigma_k)}]\bigg[ ||w_0(0) - w_0^*|| + \eta_0LE\int_0^t \eta(t') dt' \\ &+ \sqrt{h}\eta_0||\int_0^t  \hat{V}^{1/2}(t') dB(t')|| \bigg] - (w_0(t)-w_0^*)^T(\nabla F^k(w_0(t)))\\
    &\leq \mu i\eta(t)[L+\sqrt{\text{Tr}(\Sigma_k)}]\bigg[ ||w_0(0) - w_0^*|| + \eta_0LE\int_0^t \eta(t') dt' \\ &+ \sqrt{h}\eta_0||\int_0^t  \hat{V}^{1/2}(t') dB(t')|| \bigg] - \tau(F(w_0(t))-F(w_0^*)) \;\;\; (\text{Assumption \ref{wqc}}).\\
\end{align*}

Returning back to the original bound on $B_1$, we have
\begin{align*}
    B_1 &\leq \eta_0\sum_{k=1}^Q p_k\eta(t) \sum_{i=0}^{E-1}\bigg(\mu i\eta(t)[L+\sqrt{\text{Tr}(\Sigma_k)}]\bigg[ ||w_0(0) - w_0^*|| + \eta_0LE\int_0^t \eta(t') dt' + \sqrt{h}\eta_0||\int_0^t  \hat{V}^{1/2}(t') dB(t')|| \bigg] \\ &- \tau(F(w_0(t))-F(w_0^*))\bigg).
\end{align*}

Now we examine $B_2$ as 
\begin{align}
    B_2 &= \frac{1}{2}\text{Tr}(h(\eta_0(t))^2\hat{V}(t)\partial_{w_0w_0}(\mathcal{E}(w_0(t))))\\
    &= \frac{h\eta_0^2}{2}\text{Tr}(\hat{V}(t)).
\end{align}
From Lemma 3 in the supplementary materials of \cite{NEURIPS2019_9cd78264} and the same approach that we use in the proof of Theorem \ref{conv_proof_1}, we have that $\text{Tr}(\hat{V}(t)) \leq d\eta(t)^2V^*$ where $d$ is the dimensionality of our weights $w_0$. Therefore we have
\begin{align}
    B_2 &\leq \frac{dh\eta_0^2\eta(t)^2V^*}{2}.
\end{align}

We return to the bound on the infinitesimal generator as
\begin{align*}
\mathscr{A}\mathcal{E}(w_0(t)) &\leq \eta_0\sum_{k=1}^Q p_k\eta(t) \sum_{i=0}^{E-1}\bigg(\mu i\eta(t)[L+\sqrt{\text{Tr}(\Sigma_k)}]\bigg[ ||w_0(0) - w_0^*|| + \eta_0LE\int_0^t \eta(t') dt' + \sqrt{h}\eta_0||\int_0^t  \hat{V}^{1/2}(t') dB(t')|| \bigg] \\ &- \tau(F(w_0(t))-F(w_0^*))\bigg) + \frac{dh\eta_0^2\eta(t)^2V^*}{2}.
\end{align*}

As in the proof of Theorem \ref{conv_proof_1}, we now use Dynkin's Formula to get
\begin{align*}
    \mathbb{E}[\mathcal{E}(w_0(t))] - \mathcal{E}(w_0(0)) &\leq \mathbb{E}\Bigg[\int_0^t\bigg[ \eta_0\sum_{k=1}^Q p_k\eta(s) \sum_{i=0}^{E-1}\bigg(\mu i\eta(s)[L+\sqrt{\text{Tr}(\Sigma_k)}]\bigg[ ||w_0(0) - w_0^*|| + \eta_0LE\int_0^s \eta(t') dt'\\ &+ \sqrt{h}\eta_0||\int_0^s  \hat{V}^{1/2}(t') dB(t')|| \bigg] - \tau(F(w_0(s))-F(w_0^*))\bigg) + \frac{dh\eta_0^2\eta(s)^2V^*}{2} \bigg]ds\Bigg]\\
    &\leq \mathbb{E}\Bigg[\int_0^t\bigg[ \eta_0\eta(s) \bigg(\underbrace{\mu E^2\sum_{k=1}^Q p_k[L+\sqrt{\text{Tr}(\Sigma_k)}]}_{C_2}\eta(s)\bigg[ ||w_0(0) - w_0^*|| + \eta_0LE\int_0^s \eta(t') dt'\\ &+ \sqrt{h}\eta_0||\int_0^s  \hat{V}^{1/2}(t') dB(t')|| \bigg] - \tau(F(w_0(s))-F(w_0^*))\bigg) + \frac{dh\eta_0^2\eta(s)^2V^*}{2} \bigg]ds\Bigg]\\
    &= \mathbb{E}\Bigg[\int_0^t\bigg[ \eta_0C_2\eta(s)^2\bigg(\eta_0LE\int_0^s \eta(t') dt'+ \sqrt{h}\eta_0||\int_0^s  \hat{V}^{1/2}(t') dB(t')|| \bigg) \bigg]ds\Bigg]
    \\&+ \underbrace{\bigg[ \frac{dh\eta_0^2V^*}{2} + \eta_0C_2||w_0(0) - w_0^*|| \bigg] }_{C_3}\mathbb{E}\Bigg[\int_0^t\eta(s)^2 ds \Bigg] - \mathbb{E}\Bigg[\int_0^t\bigg[\eta_0\eta(s) \tau(F(w_0(s))-F(w_0^*)) \bigg]ds \Bigg]\\
    &= \eta_0^2C_2\mathbb{E}\Bigg[\int_0^t\bigg[\eta(s)^2\bigg(LE\int_0^s \eta(t') dt'+ \sqrt{h}||\int_0^s  \hat{V}^{1/2}(t') dB(t')|| \bigg) \bigg]ds\Bigg]
    \\&+ C_3 \mathbb{E}\Bigg[\int_0^t\eta(s)^2 ds \Bigg] - \tau\eta_0\mathbb{E}\Bigg[\int_0^t\bigg[\eta(s) (F(w_0(s))-F(w_0^*)) \bigg]ds \Bigg].
\end{align*}

We examine the term with $\hat{V}^{1/2}(t')$ as
\begin{align*}
    \mathbb{E}\Bigg[\int_0^t \eta(s)^2\sqrt{h}||\int_0^s  \hat{V}^{1/2}(t') dB(t')|| ds\Bigg] = \int_0^t \eta(s)^2\sqrt{h}\mathbb{E}\bigg[||\int_0^s  \hat{V}^{1/2}(t') dB(t')||\bigg] ds,\\
\end{align*}
and we can make this switch because $\eta(s)^2\sqrt{h}\mathbb{E}\bigg[||\int_0^s  \hat{V}^{1/2}(t') dB(t')||\bigg]< \infty$. From multivariate Ito isometry and Jensen's inequality for concave functions (such as the square root function), we have that 
\begin{align*}
    \mathbb{E}||\int_0^s \hat{V}^{1/2}(t') dB(t')|| &\leq \sqrt{\mathbb{E}||\int_0^s \hat{V}^{1/2}(t') dB(t')||^2}\\
    &\leq \sqrt{\mathbb{E}\int_0^s ||\hat{V}^{1/2}(t')||^2 dt'}\\
    &= \sqrt{\mathbb{E}\int_0^s ||\eta(t')\hat{V_1}^{1/2}(t')||_F^2 dt'}\\
    &= \sqrt{\mathbb{E}\int_0^s \eta(t')^2||\hat{V_1}^{1/2}(t')||_F^2 dt'}\\
    &\leq \sqrt{\int_0^s \eta(t')^2(V^*)^2 dt'}\\
    &= V^*\sqrt{\int_0^s \eta(t')^2 dt'},\\
\end{align*}
where $||\cdot||_F$ is the Frobenius norm.

Returning to the bound from Dynkin's Formula we have
\begin{align*}
    \underbrace{\mathbb{E}[\mathcal{E}(w_0(t))]}_{\geq 0 } - \mathcal{E}(w_0(0)) &\leq  \eta_0^2C_2 \int_0^t\bigg[\eta(s)^2\bigg(LE\int_0^s \eta(t') dt'+ \sqrt{h}V^*\sqrt{\int_0^s \eta(t')^2 dt'} \bigg) \bigg]ds\\
    & + C_3 \int_0^t \eta(s)^2 ds - \tau\eta_0\mathbb{E}\Bigg[\int_0^t\bigg[\eta(s) (F(w_0(s))-F(w_0^*)) \bigg]ds \Bigg].
\end{align*}

Rearranging we obtain
\begin{align*}
     \mathbb{E}\Bigg[\int_0^t\bigg[\eta(s) (F(w_0(s))-F(w_0^*)) \bigg]ds \Bigg] &\leq \frac{\mathcal{E}(w_0(0))}{\tau\eta_0} + \frac{\eta_0^2C_2}{\tau\eta_0} \int_0^t\bigg[\eta(s)^2\bigg(LE\int_0^s \eta(t') dt'+ \sqrt{h}V^*\sqrt{\int_0^s \eta(t')^2 dt'} \bigg) \bigg]ds\\
    & + \frac{C_3}{\tau\eta_0} \int_0^t \eta(s)^2 ds.
\end{align*}

Using the same trick as in the proof of Theorem \ref{conv_proof_1}, where we create random variable $\hat{t} \in [0,t]$ with probability density function $\frac{\eta(t')}{\varphi(t)}$, we can substitute
\begin{equation*}
\varphi(t)\mathbb{E}_{\tilde{t}}[(F(w_0(\hat{t}))-F(w_0^*))] = \int_0^t \eta(t')(F(w_0(t'))-F(w_0^*)) dt',
\end{equation*} where $\varphi(t) = \int_{0}^t \eta(t')dt'$.

We finally find
\begin{align*}
     \mathbb{E}_{\mathcal{G},\tilde{t}}[(F(w_0(\hat{t}))-F(w_0^*))] &\leq \frac{||w_0(0)-w_0^*||}{\tau\eta_0\varphi(t)} + \frac{\eta_0^2C_2}{\tau\eta_0\varphi(t)} \int_0^t\bigg[\eta(s)^2\bigg(LE\int_0^s \eta(t') dt'+ \sqrt{h}V^*\sqrt{\int_0^s \eta(t')^2 dt'} \bigg) \bigg]ds\\
    & + \frac{C_3}{\tau\eta_0\varphi(t)} \int_0^t \eta(s)^2 ds.
\end{align*}

\end{proof}

\subsection{Proof of Corollary \ref{cor:concrete_rates_wqc}}
\label{proof_cor_wqc}
\begin{proof}
We have $\varphi(t) = \log(t+1)$.

We examine the case where $\eta(t)=\frac{1}{t+1}$ as
\begin{align*}
     \mathbb{E}_{\mathcal{G},\tilde{t}}[(F(w_0(\hat{t}))-F(w_0^*))] &\leq \frac{||w_0(0)-w_0^*||}{\tau\eta_0\varphi(t)} + \frac{\eta_0^2C_2}{\tau\eta_0\varphi(t)} \int_0^t\bigg[\eta(s)^2\bigg(LE\int_0^s \eta(t') dt'+ \sqrt{h}V^*\sqrt{\int_0^s \eta(t')^2 dt'} \bigg) \bigg]ds\\
    & + \frac{C_3}{\tau\eta_0\varphi(t)} \int_0^t \eta(s)^2 ds \\
    &= \frac{||w_0(0)-w_0^*||}{\tau\eta_0\log(t+1)} + \frac{\eta_0^2C_2}{\tau\eta_0\log(t+1)} \int_0^t\bigg[\bigg(LE\frac{\log(s+1)}{(s+1)^2}+ \frac{\sqrt{h}V^*}{(s+1)^2}\sqrt{\frac{s}{(s+1)}} \bigg) \bigg]ds\\
    & + \frac{C_3}{\tau\eta_0\log(t+1)} \int_0^t \eta(s)^2 ds \\
    &\leq \frac{||w_0(0)-w_0^*|| + C_3 + \eta_0^2C_2\sqrt{h}V^*}{\tau\eta_0\log(t+1)} + \frac{\eta_0^2C_2LE }{\tau\eta_0}\cdot\frac{t-\log(t+1)}{t\log(t+1)}.
\end{align*}
\end{proof}

\subsection{Proof of Theorem \ref{lyapunov_thm}}
\label{lyapunov_proof}
\begin{proof}
\textbf{Lyapunov CLT}
    We first describe the Lyapunov Central Limit Theorem. Assume we have a sequence of random variables $X_1,X_2,...,X_n$ which are independent but not necessarily identically distributed. These random variables have possible different means $\mu_i$ and variances $\sigma_i^2$. To satisfy the Lyapunov condition, we need to show that for some $\delta>0$
    \begin{equation}
        \label{lyapunov_condition}
        \lim_{n \rightarrow \infty} \frac{1}{s_n^{2+\delta}}\sum_{i=1}^n \mathbb{E}[|X_i - \mu_i|^{2+\delta}]=0
    \end{equation} where $s_n^2 = \sum_{i=1}^n \sigma_i^2$.
    
    The problem we are studying is of $A(t)$, which is formulated as 
    \begin{equation}
        A(t)=\sum_{k=1}^{Q} \underbrace{p_k [\sum_{i=0}^{E-1} \eta_k(t)(N^k_{t-E+i}-G^k_{t-E+i})]}_{A_k}
    \end{equation} where $N^k_{t+i}\sim \mathcal{N}(0,\Sigma_k(w_{t+i}^k))$ and $G_{t+i}^k = \nabla F^k(w_{t+i}^k)$.

    For simplicity, the proof below assumes $A_k$ is one-dimensional. However, from Assumption \ref{assumption:diag_covariance}, we have that $\Sigma_k(w_{t+i}^k)$ is diagonal. This means we can apply this approach to each component of a vector $A_k$, and thus each component of $A_k$ tends to a normal distribution that are independent of each other because $\Sigma_k(w_{t+i}^k)$ is diagonal. Since the components are independent and normally distributed, we have that $A_k$ tends to a multivariate normal distribution. 
    
    To satisfy the Lyapunov condition, we need to show that for some choice of $\delta$ the following holds
    \begin{align*}
        \lim_{Q \rightarrow \infty} \frac{1}{s_Q^{2+\delta}}\sum_{k=1}^Q \mathbb{E}\bigg[\big|A_k - \mathbb{E}[A_k]\big|^{2+\delta}\bigg]=0.
    \end{align*}
    We choose $\delta=2$ and examine $\mathbb{E}\bigg[\big|A_k - \mathbb{E}[A_k]\big|^{2+\delta}\bigg]$ as
    \begin{align*}
        \mathbb{E}\bigg[\big|A_k - \mathbb{E}[A_k]\big|^{2+\delta}\bigg] &= \mathbb{E}\bigg[\Big|p_k [\sum_{i=0}^{E-1} \eta_k(t)(N^k_{t-E+i}-G^k_{t-E+i})] - \mathbb{E}[p_k [\sum_{i=0}^{E-1} \eta_k(t)(N^k_{t-E+i}-G^k_{t-E+i})]]\Big|^{4}\bigg]\\
        &= p_k^4\eta_k(t)^4\mathbb{E}\bigg[\Big| \underbrace{\sum_{i=0}^{E-1} (N^k_{t-E+i})}_{=N_k} + \underbrace{\sum_{i=0}^{E-1} (\mathbb{E}[G^k_{t-E+i}] - G^k_{t-E+i})}_{=R_k}\Big|^{4}\bigg]\\
        &= p_k^4\eta_k(t)^4\mathbb{E}\bigg[\Big( N_k + R_k \Big)^{4}\bigg]\\
        &= p_k^4\eta_k(t)^4\mathbb{E}\bigg[ N_k^4 + 4N_k^3R_k + 6N_k^2R_k^2 +4N_kR_k^3+R_k^4 \bigg]\\
        &= p_k^4\eta_k(t)^4\bigg[ \mathbb{E}[N_k^4] + 4\mathbb{E}[N_k^3R_k] + 6\mathbb{E}[N_k^2R_k^2] +\mathbb{E}[4N_kR_k^3]+\mathbb{E}[R_k^4] \bigg].
    \end{align*}

    We have that $s_Q^4=\bigg(\sum_{k=1}^Q\mathbb{E}\bigg[\big(A_k - \mathbb{E}[A_k]\big)^{2}\bigg]\bigg)^2=\Big(\sum_{k=1}^Q p_k^2\eta_k(t)^2\mathbb{E}\big[(N_k+R_k)^2\big] \Big)^2$.
    
    We can use the following formula which states that for positive $x_i$,
    \begin{equation*}
        \sum_{i=1}^n x_i^2 = \frac{(\sum_{i=1}^n x_i)^2 + \frac{1}{2}\sum_{i=1}^n\sum_{j=1}^n(x_i-x_j)^2}{n}.
    \end{equation*}
    We quickly prove this equation as follows
    \begin{align*}
        \frac{(\sum_{i=1}^n x_i)^2 + \frac{1}{2}\sum_{i=1}^n\sum_{j=1}^n(x_i-x_j)^2}{n} &= \frac{\sum_{i=1}^n x_i^2 + \sum_{i=1}^n\sum_{j\neq i} x_ix_j + \frac{1}{2}\sum_{i=1}^n\sum_{j=1}^n(x_i^2-2x_ix_j+x_j^2)}{n}\\
        &= \frac{(n+1)\sum_{i=1}^n x_i^2 + \sum_{i=1}^n\sum_{j\neq i} x_ix_j - \sum_{i=1}^n\sum_{j=1}^nx_ix_j}{n}\\
        &= \frac{(n+1)\sum_{i=1}^n x_i^2 + \sum_{i=1}^n(\sum_{j\neq i} x_ix_j - \sum_{j=1}^nx_ix_j)}{n}\\
        &= \frac{(n+1)\sum_{i=1}^n x_i^2 - \sum_{i=1}^n x_i^2}{n}\\
        &= \sum_{i=1}^n x_i^2.
    \end{align*}

    Using this formula and substituting $x_i = p_k^2\eta_k(t)^2\mathbb{E}\big[(N_k+R_k)^2\big]$ we get 
    \begin{align*}
        \Big(\sum_{k=1}^Q p_k^2\eta_k(t)^2\mathbb{E}\big[(N_k+R_k)^2\big] \Big)^2 &= Q\sum_{k=1}^Q \bigg(p_k^2\eta_k(t)^2\mathbb{E}\big[(N_K+R_k)^2\big]\bigg)^2\\ &- \frac{1}{2}\sum_{i=1}^Q\sum_{j=1}^Q\bigg(p_i^2\eta_i(t)^2\mathbb{E}\big[(N_i+R_i)^2\big] - p_j^2\eta_j(t)^2\mathbb{E}\big[(N_j+R_j)^2\big]\bigg)^2.
    \end{align*}

    From our assumptions on the similarity of variance between clients, we have for some $C>0$
    \begin{align*}
        \Big(\sum_{k=1}^Q p_k^2\eta_k(t)^2\mathbb{E}\big[(N_K+R_k)^2\big] \Big)^2 &= Q\sum_{k=1}^QC\\
        &= CQ^2.
    \end{align*}

    We have
    \begin{align*}
        \frac{1}{s_Q^{4}}\sum_{k=1}^Q \mathbb{E}\bigg[\big|A_k - \mathbb{E}[A_k]\big|^{4}\bigg] &= \frac{1}{CQ^2}\sum_{k=1}^Q \mathbb{E}\bigg[\big|A_k - \mathbb{E}[A_k]\big|^{4}\bigg]\\
        &= \frac{1}{CQ^2}\sum_{k=1}^Q p_k^4\eta_k(t)^4\bigg[ \mathbb{E}[N_k^4] + 4\mathbb{E}[N_k^3R_k] + 6\mathbb{E}[N_k^2R_k^2] +\mathbb{E}[4N_kR_k^3]+\mathbb{E}[R_k^4] \bigg]\\
        &\leq \frac{1}{CQ^2}\sum_{k=1}^Q p_k^4\eta_k(t)^4\bigg[ |\mathbb{E}[N_k^4]| + 4|\mathbb{E}[N_k^3R_k]| + 6|\mathbb{E}[N_k^2R_k^2]| +4|\mathbb{E}[N_kR_k^3]|+|\mathbb{E}[R_k^4]| \bigg]\\
        &\leq \frac{1}{CQ^2}\sum_{k=1}^Q p_k^4\eta_k(t)^4\bigg[ D + 4D + 6D + 4D + D \bigg]\\
        &\leq \frac{16DQ}{CQ^2}.
    \end{align*}
    
    Since, $0\leq \frac{1}{s_Q^{4}}\sum_{k=1}^Q \mathbb{E}\bigg[\big|A_k - \mathbb{E}[A_k]\big|^{4}\bigg] \leq \frac{16D}{CQ}$, and $\lim_{Q \rightarrow \infty}0=0$ and $\lim_{Q \rightarrow \infty}\frac{16D}{CQ}=0$, by the squeeze theorem, we know that $\lim_{Q \rightarrow \infty}\frac{1}{s_Q^{4}}\sum_{k=1}^Q \mathbb{E}\bigg[\big|A_k - \mathbb{E}[A_k]\big|^{4}\bigg]=0$. Thus, the Lyapunov condition holds.
    
\end{proof}

\subsection{Proof of Theorem \ref{thm-localmins}}
\label{proof:localmins}
\begin{proof}
We wish to show that after $E$ local iterations of SGD, the local weights evolve as
\begin{equation*}
    w_{t+E}^k \sim w_t^0 - \eta \sum_{j=0}^{E-1}(I-\eta U_k)^jU_k(w_t^0-a_k) - \eta \sum_{j=0}^{E-1}\sum_{i=1}^j (I-\eta U_k)^{j-i}U_k\mathcal{N}(0,\Sigma_k) + E\mathcal{N}(0,\Sigma_k).
\end{equation*}

First observe the evolution for a few iterations as
\begin{align*}
    w_{t+1}^k &\sim w_t^0 -\eta \mathcal{N}( U_k(w_t^0-a_k),\Sigma_k)\\
    w_{t+2}^k | w_{t+1}^k &\sim w_{t+1}^k -\eta \mathcal{N}( U_k(w_{t+1}^k-a_k),\Sigma_k)\\
    &\sim w_t^0 - \eta U_k(w_t^0-a_k)-\eta U_k(w_{t+1}^k-a_k)+2\eta\mathcal{N}(0,\Sigma_k)\\
    w_{t+3}^k|(w_{t+2}^k,w_{t+1}^k) &\sim w_t^0 -\eta U_k(w_t^0-a_k)-\eta U_k(w_{t+1}^k-a_k) +2\eta\mathcal{N}(0,\Sigma_k)-\eta U_k(w_{t+2}^k-a_k)+\eta\mathcal{N}(0,\Sigma_k).
\end{align*}
Extend to $E$ number of time steps forward as
\begin{align}
\label{general_1}
    w_{t+E}^k|(w_{t+1}^k,...,w_{t+E-1}^k) \sim w_t^0 - \eta \sum_{j=0}^{E-1} U_k(w_{t+j}^k - a_k) + E\eta\mathcal{N}(0,\Sigma_k).
\end{align}
In general we have
\begin{align*}
U_k(w_{t+j}^k-a_k) &= (1-\eta U_k)U_k(w_{t+j-1}^k-a_k)+U_k\mathcal{N}(0,\Sigma_k).
\end{align*}
Starting from $j=1$, we have
\begin{align*}
U_k(w_{t+1}^k-a_k) &= (1-\eta U_k)U_k(w_{t}^0-a_k)+U_k\mathcal{N}(0,\Sigma_k)\\
&= (U_k-\eta U_k^2)w_t^0+(-U_k+\eta U_k^2)a_k+U_k\mathcal{N}(0,\Sigma_k).
\end{align*}
Now for $j=2$, we have
\begin{align*}
U_k(w_{t+2}^k-a_k) &= (1-\eta U_k)[(1-\eta U_k)U_k(w_t^0-a_k)+U_k\mathcal{N}(0,\Sigma_k)]+U_k\mathcal{N}(0,\Sigma_k).
\end{align*}

Expanding out to $j=E$, we have
\begin{align}
\label{general_2}
U_k(w_{t+j}^k-a_k) = (I-\eta U_k)^jU_k(w_t^0-a_k)+\sum_{i=1}^j(I-\eta U_k)^{j-i}U_k\mathcal{N}(0,\Sigma_k).
\end{align}

Combining equations \ref{general_1} and \ref{general_2}, we have
\begin{align*}
w_{t+E}^k \sim w_t^0 - \eta \sum_{j=0}^{E-1}(I-\eta U_k)^j U_k (w_t^0 - a_k) - \eta \sum_{j=0}^{E-1}\sum_{i=1}^j (I-\eta U_k)^{j-i}U_k\mathcal{N}(0,\Sigma_k) + E\eta\mathcal{N}(0,\Sigma_k).
\end{align*}

\end{proof}

\subsection{Global loss function form for quadratic assumption}
\label{rewrite_global_loss}
With the quadratic assumption, we can reformat the equation for the server loss function as
\begin{align*}
    F(w) = \frac{1}{2}(w-a)^T(\sum_{k=1}^Qp_kU_k)(w-a).
\end{align*}
We must find $a$. We have that $\min_w F(w)=a$, $\nabla F(w)=\sum_{k=1}^Q p_kU_k(w-a_k)$ and $\nabla_w F(a)=0$.
We find
\begin{align*}
\nabla_w F(a)=0&=\sum_{k=1}^Q p_kU_k(a-a_k)\\
&=\sum_{k=1}^Q p_kU_ka -\sum_{k=1}^Q p_kU_ka_k.
\end{align*}
Now we solve for $a$ as
\begin{align*}
\sum_{k=1}^Q p_kU_ka =\sum_{k=1}^Q p_kU_ka_k\\
a = (\sum_{k=1}^Qp_kU_k)^{-1}\sum_{k=1}^Qp_kU_k a_k.
\end{align*}

Therefore, the server loss function is also represented by a quadratic form with local minimum at $w_0^*=(\sum_{k=1}^Qp_kU_k)^{-1}\sum_{k=1}^Qp_kU_k a_k$ and Hessian $H_0 = (\sum_{k=1}^Qp_kU_k)$.

\subsection{Proof of Theorem \ref{quadratic_approx_sol}}
\label{proof_quadratic_approx_sol}
For a linear SDE of the form, $dX(t)=(a(t) + A(t)X(t))dt + b(t)dB(t)$, the solution is normally distributed as $\mathcal{N}(m(t),v(t))$. The mean $m(t)$ is the solution of the ordinary differential equation (ODE) $\frac{dm(t)}{dt} = A(t)m(t) + a(t)$ with initial condition $m(0)=X(0)$. The variance $v(t)$ is the solution of the ODE $\frac{dv(t)}{dt}=2A(t)v(t)+b(t)^2$ with initial condition $v(0)=0$ \cite{linearSDEs}. These ODEs are straightforward to solve and we obtain the solution in Theorem \ref{quadratic_approx_sol}.

\end{document}